\documentclass[sigconf]{acmart}

\AtBeginDocument{%
  }

\copyrightyear{2023} 
\acmYear{2023} 
\setcopyright{rightsretained} 
\acmConference[AIES '23]{AAAI/ACM Conference on AI, Ethics, and Society}{August 8--10, 2023}{Montréal, QC, Canada}
\acmBooktitle{AAAI/ACM Conference on AI, Ethics, and Society (AIES '23), August 8--10, 2023, Montréal, QC, Canada}
\acmDOI{10.1145/3600211.3604676}
\acmISBN{979-8-4007-0231-0/23/08}

\usepackage[utf8]{inputenc}

\usepackage{graphicx}
\graphicspath{ {images/} }
\usepackage{subcaption}

\usepackage{amsmath} 
\usepackage{mathtools} 
\usepackage{amsthm} 
\usepackage{bm} 
\usepackage{bbm} 

\usepackage{enumitem} 

\usepackage{tikz}
\usepackage{tikz-cd} 
\usetikzlibrary{positioning}

\usepackage{hyperref} 
\usepackage[capitalize,noabbrev]{cleveref}
\usepackage{url} 

\usepackage{soul}

\usepackage{booktabs} 
\usepackage{multirow} 

\usepackage{xspace} 

\usepackage{algorithm} 
\usepackage{algpseudocode} 

\usepackage[colorinlistoftodos,prependcaption,textsize=small]{todonotes} 

\usepackage{widetext}

\definecolor{cadmiumgreen}{rgb}{0.0, 0.40, 0.25}
\definecolor{lightcadmiumgreen}{rgb}{0.0, 0.60, 0.30}
\definecolor{cadmiumorange}{rgb}{0.93, 0.53, 0.18}
\definecolor{burgundy}{rgb}{0.50, 0.0, 0.13}
\definecolor{airforceblue}{rgb}{0.36, 0.54, 0.66}
\definecolor{niceblue}{rgb}{0.13, 0.33, 0.83}
\definecolor{nicegreen}{rgb}{0.13, 0.55, 0.17}

\iffalse    
    \newcommand{\del}[1]{\colorlet{saved}{.}\color{olive}\st{#1}\color{saved}\xspace}
    
    \newcommand{\itodo}[1]{\colorlet{saved}{.}\color{magenta}\textbf{TODO}: {#1} \color{saved}\xspace}
    \newcommand{\todox}[1]{\todo[inline,linecolor=magenta,backgroundcolor=magenta!25,bordercolor=magenta]{\color{magenta}\textbf{TODO}: \color{black} #1}}
    
    \newcommand{\EA}[1]{\colorlet{saved}{.}\color{cadmiumgreen}{#1}\color{saved}\xspace} 
\else
    \newcommand{\del}[1]{}
    
    \newcommand{\itodo}[1]{}
    \newcommand{\todox}[1]{}
    
    \newcommand{\EA}[1]{{#1}\xspace}
\fi

\theoremstyle{plain}
\newtheorem{theorem}{Theorem}[section]
\newtheorem{proposition}[theorem]{Proposition}

\newtheorem{corollary}[theorem]{Corollary}
\theoremstyle{definition}
\newtheorem{definition}[theorem]{Definition}

\theoremstyle{remark}

\makeatletter
\newcommand{\customlabelnref}[2]{\protected@write \@auxout {}{\string \newlabel {#1}{{#2}{\thepage}{#2}{#1}{}}}\hypertarget{#1}{#2}}
\newcommand{\customlabel}[2]{\protected@write \@auxout {}{\string \newlabel {#1}{{#2}{\thepage}{#2}{#1}{}}}}
\makeatother

\newcommand{\norm}[1]{\left\lVert#1\right\rVert} 

\newcommand{\card}[1]{\left|#1\right|}
\newcommand{\set}[1]{\left\{#1\right\}}

\newcommand{\lrsquare}[1][]{\left[ #1 \right]}
\newcommand{\lrround}[1]{\left( #1 \right)}
\newcommand{\lrangle}[1]{\left\langle #1 \right\rangle}
\newcommand{\round}{\lrround}

\newcommand{\tuple}{\lrangle}

\DeclareMathOperator*{\Exp}{\mathbb{E}}
\makeatletter
\def\E_#1{\Exp_{#1}\@ifnextchar[{\lrsquare}{}} 
\makeatother

\newcommand{\Ind}{\mathbbm{1}}
\makeatletter
\def\I{\Ind\!\@ifnextchar[{\lrsquare}{}} 
\makeatother

\newcommand{\R}{\mathbb{R}}

\newcommand{\x}{\bm{x}}
\newcommand{\xp}{\bm{x'}}
\newcommand{\Xp}{\X'}
\newcommand{\X}{\mathcal{X}}

\newcommand{\F}{\mathcal{F}}
\newcommand{\m}{m}

\newcommand{\Rm}{\mathbb{R}^{\m}}

\newcommand{\phihat}{\hat{\Phi}}
\newcommand{\phiv}{\bm{\Phi}}
\newcommand{\phivhat}{\bm{\phihat}}

\newcommand{\psihat}{\hat{\Psi}}
\newcommand{\psiv}{\bm{\Psi}}
\newcommand{\psivhat}{\bm{\psihat}}

\newcommand{\gammav}{\bm{\Gamma}}

\newcommand{\MS}{\text{M}\!\text{S}}
\newcommand{\WMS}{\text{W}\!\text{M}\!\text{S}}

\newcommand{\C}{C}
\newcommand{\notC}{\bar{C}}

\makeatletter
\newcommand{\spaceornewline}[1]{%
\if@twocolumn%
$$
$$
\else
#1
\fi
}%
\makeatother

\begin{document}

\title{On the Connection between Game-Theoretic Feature Attributions and Counterfactual Explanations}

\author{Emanuele Albini}
\orcid{0000-0003-2964-4638}
\affiliation{%
  \institution{J.P. Morgan AI Research}
  \city{London}
  \country{UK}
}
\email{emanuele.albini@jpmorgan.com}

\author{Shubham Sharma}
\affiliation{%
  \institution{J.P. Morgan AI Research}
  \city{New York}
  \country{USA}
}
\email{shubham.x2.sharma@jpmorgan.com}

\author{Saumitra Mishra}
\affiliation{%
  \institution{J.P. Morgan AI Research}
  \city{London}
  \country{UK}
}
\email{saumitra.mishra@jpmorgan.com}

\author{Danial Dervovic}
\affiliation{%
  \institution{J.P. Morgan AI Research}
  \city{New York}
  \country{USA}
}
\email{danial.dervovic@jpmorgan.com}

\author{Daniele Magazzeni}
\affiliation{%
  \institution{J.P. Morgan AI Research}
  \city{London}
  \country{UK}
}
\email{daniele.magazzeni@jpmorgan.com}

\renewcommand{\shortauthors}{Albini et al.}

\begin{abstract}
Explainable Artificial Intelligence (XAI) has received widespread interest in recent years, and two of the most popular types of explanations are feature attributions, and counterfactual explanations. 
These classes of approaches have been largely studied independently and the few attempts at reconciling them have been primarily empirical. 
This work establishes a clear theoretical connection between game-theoretic feature attributions, focusing on but not limited to SHAP, and counterfactuals explanations.
After motivating operative changes to Shapley values based feature attributions and counterfactual explanations, we prove that, under conditions, they are in fact equivalent. We then extend the equivalency result to game-theoretic solution concepts beyond Shapley values. 
Moreover, through the analysis of the conditions of such equivalence, we shed light on the limitations of naively using counterfactual explanations to provide feature importances.
Experiments on three datasets quantitatively show the difference in explanations at every stage of the connection between the two approaches and corroborate the theoretical findings.
\end{abstract}




\maketitle
\section{Introduction}\label{sec:introduction}

As complex machine learning models are used extensively in industry settings, including in numerous high-stakes domains such as finance \citep{Veloso2021,Bhatore2020}
healthcare \citep{Yu2018,Markus2021} 
and autonomous driving \citep{Grigorescu2020}, explaining the outcomes of such models has become, in some cases, a legal requirement \citep{Jobin2019}, e.g., U.S. Equal Opportunity Act \citep{USRegulationB} and E.U. General Data Protection Regulation \citep{GDPR2016}.
The use of XAI techniques is increasingly becoming a standard practice at every stage of the lifecycle of a model \citep{Bhatt2020}: during development, to debug the model and increase its performance; during review, to understand the inner working mechanisms of the model; and in production, to monitor its effectiveness \citep{Mohseni2021}.

In this context, two different classes of approaches have received a lot of attention from the research community in the last few years: \emph{feature attribution} techniques and \emph{counterfactual explanations}.

Feature attributions aim at distributing the output of the model for a specific input to its features. To accomplish this, they compare the output of the (same) model when a feature is present with that of when the same feature is remove, e.g., \citep{Ribeiro2016,Lundberg2017,Sundararajan2017}.

\EA{Counterfactual explanations instead} aim to answer the question: what would have to change in the input to change the outcome of the model \citep{Wachter2017}. Towards this goal, desirable properties of the modified input, also known as the "counterfactual", are: proximity, realism, and sparsity with respect to the input \citep{Barocas2020,Keane2021}.



As these two explanation types are used to understand models, an imperative question is: \emph{"How do feature attribution based explanations and counterfactual explanations align with each other?"} Unifying counterfactual explanations with feature attributions techniques is an open question \citep{Verma2020}. In fact, while counterfactual explanations aim to provide users with ways to change a model decision \citep{Venkatasubramanian2020}, it has been argued that they do not fulfil the normative constraints of identifying the principal reasons for a certain outcome, as feature attributions do \citep{Selbst2018}.

Although these two classes of approaches have largely been studied in isolation, there has been some work (primarily empirical) to address a connection between the two:

\begin{itemize}
\item When motivating the usefulness of counterfactual explanations, researchers have drawn attention on how a set of counterfactual points can be used to directly generate \textit{feature importances based on how frequently features are modified in counterfactuals} \citep{Sharma2020, Barocas2020, Mothilal2021}. 
\item On the feature attribution side, a recent line of research has been gaining traction around combining counterfactuals and Shapley values with the goal to generate \textit{feature attributions with a counterfactual flavour \EA{by using counterfactuals as background distributions for SHAP explanations}} \citep{Albini2022, Lahiri2022}.
\end{itemize}

However, there has been no work that establishes a clear theoretical connection between these approaches, \EA{or that theoretically analyses their limitations and assumptions}. 

This paper bridges the gap between these two lines of research that have been developing in parallel: 

\begin{itemize}
\item We provide and justify operative changes to the \emph{counterfactual frequency-based feature importance} and \emph{Shapley values-based feature attributions} that are necessary to make an equivalency statement between the two explanations.
\item We theoretically prove that --- after imposing some conditions on the counterfactuals --- the \emph{Shapley values-based feature attribution} and the \emph{counterfactual frequency-based feature importance} are equivalent.
\item We discuss what are the effects of such an equivalency, with particular attention to (1) the game-theoretic interpretation of explanations and (2) the limitations of \emph{counterfactual frequency-based feature importance} in providing a detailed account of the importance of the features.
\item \EA{We generalise the connection with \emph{counterfactual frequency-based feature importance} to a wider range of game-theoretic solution concepts beyond Shapley values.}
\item We perform an ablation study to show how each of the proposed operative changes (required to establish equivalency) will impact the explanations, \EA{and we show how the empirical results are coherent with the theoretical findings}.
\item \EA{Finally, we evaluate these explanations} using common metrics from the XAI literature \EA{as necessity, sufficiency, plausibility and counterfactual-ability \citep{Mothilal2021,Albini2022}, and once again we show how the empirical results are coherent with the theoretical findings.}
\end{itemize}

\EA{It is important to note that the theory established in this paper applies to \emph{any} counterfactual explanation, \emph{independently of the technique} used for its generation, and is also valid when considering \emph{multiple} or \emph{diverse} counterfactual explanations for the same query instance \citep{Mothilal2020a,Dandl2020,Papantonis2022,Smyth2021}}.

\section{Background}\label{sec:background}

Consider a classification \emph{model} $f : \Rm \rightarrow \R$ and its \emph{decision function} $F : \X \rightarrow \{0, 1\}$ with \emph{threshold} $t \in \R$:
$$F(\x) = \begin{cases} 1 & \text{if } f(\x) > t\\ 0 & \text{otherwise}\end{cases}.$$

We refer to $f(\x)$ as the model \emph{output} and to $F(\x)$ as the model \emph{prediction}.
Without loss of generality, in the remainder of this paper we assume that $\x$ is such that $F(\x) = 1$.

\subsection{Counterfactual Explanations}\label{sec:counterfactuals}

A \emph{counterfactual} \citep{Wachter2017,Karimi2022}
for a \textit{query instance} $\x \in \Rm$ is a point $\xp \in \Rm$ such that:
(1) $\xp$ is \emph{valid}, i.e., $F(\xp) \neq F(\x)$;
(2) $\xp$ is \emph{close} to $\x$ (under some metric);
(3) $\xp$ is a \emph{plausible} input.

The plausibility requirement has taken different forms. 
It may involve considerations about proximity to the data manifold \citep{Pawelczyk2020,Kenny2021}, proximity to other counterfactuals \citep{Laugel2019}, causality \citep{Karimi2020a}, actionability \citep{Ustun2019,Poyiadzi2020}, 
robustness \citep{Upadhyay,Pawelczyk2022a,Sharma2022} 
or a combination thereof \citep{Dandl2020}.

Another key aspect for counterfactual explanations is their \emph{sparsity} \citep{Fernandez2019,Smyth2021,Rodriguez2021,
Sharma2022,Lang2022}.
Optimising for sparsity forces explanations (1) to ignore features that are not used by the model to make decisions, and (2) in general, to be more concise, 
as advocated also from a social science perspective \citep{Miller2019}.
However, criticisms about sparsity have been raised, e.g., in the actionable recourse settings \citep{Pawelczyk2020,VanLooveren2021a},
as sparsity could give rise to explanations that are less plausible.
Ultimately, this argument reduces to the well-known thread-off between explanations that are ``true to the model'' (more sparse) or ``true to the data'' (more plausible) \citep{Chen2020,Janzing2020}. 

A plethora of techniques for the generation of counterfactuals exist in the literature using search algorithms \citep{Wachter2017,Spooner2021,Albini2020,Albini2021b}, 
optimisation \citep{Kanamori2022} and genetic algorithms \citep{Sharma2020} among other methods. We refer the reader to recent surveys for more details \citep{Keane2021,Stepin2021,Guidotti2022,Karimi2020}.

Few authors have suggested to generate a feature importance from counterfactual explanations \citep{Sharma2020,Barocas2020}.
In particular, 
\citet{Mothilal2021} proposed to use the fraction of counterfactual examples (for the same query instance) that have a modified value as the feature importance. 
The formal definition follows.

\begin{definition}[\textsc{CF-Freq}]\label{def:cf_freq}
    Given a query instance $\x$ and a set of counterfactuals $\Xp$ the \emph{counterfactual frequency importance}%
    \footnote{Our term. No specific name beyond the more general ``counterfactual feature importance'' had been given in the literature.},
    denoted with $\psiv$, is defined as follows. \citep{Mothilal2021}
    $$
        \psiv = \E_{\xp \sim \Xp}\left[ \I[\x \neq \xp] \right]
    $$
    where 
    $\Ind$ is the binary indicator operator.
\end{definition}

\EA{The assumption behind this \textsc{CF-Freq} feature importance is that a feature modified more often in counterfactual examples is more important than others which are changed less often. We will show in \cref{sec:feature_removal_strategy} how this assumption has an important effect on the explanation that is generated}.

\subsection{SHAP}\label{sec:attributions}\label{sec:shap}


The Shapley value is a solution concept in classic game theory used to attribute the payoff to the players in an $\m$-player cooperative game. 
Given a set of players $\F = \{ 1, \ldots, \m \}$ and the \emph{characteristic function} $v : 2^{\F} \rightarrow \R$ of a cooperative game, Shapley values are used to attribute the payoff returned by the characteristics function to the players.

\begin{definition}[Shapley values]\label{def:shapley}
    The \emph{Shapley value} for player $i$ is defined as follows. \citep{Shapley1951}
    $$
        \sum_{S \subseteq \F \setminus \{ i \}} w(|S|) \left[ v(S \cup \{i \}) - v(S) \right]
        \spaceornewline{\quad} \text{ where }\quad
        w(s) = \frac{1}{m} \binom{\m-1}{s}^{-1}        
    $$
\end{definition}


        
\begin{figure*}[!ht]
    \centering
\begin{tikzcd}[row sep=huge, column sep=huge, remember picture]
        \parbox{2.2cm}{\centering\small SHAP \\ \scriptsize\citet{Lundberg2017}}
        \arrow[r, swap, "\scriptsize\parbox{1.8cm}{\center {Counterfactuals as Background Distribution}}"] 
    &
        \parbox{2.2cm}{\centering\small CF-SHAP \\ \scriptsize\citet{Albini2022} \\ \small $\phiv$}
        \arrow[r, line width=0.5pt, swap, "\scriptsize\parbox{2cm}{\center {\textbf{Binary Prediction} (Query Function, \cref{sec:model_behaviour})}}"]
    &
        \parbox{2cm}{\centering\small Binary \\ CF-SHAP \\ $\phivhat$}
        \arrow[r, Leftrightarrow, "\scriptsize\parbox{2.75cm}{\center \textbf{Maximal Sparsity of Counterfactuals}\\(\cref{theo:connection,theo:connection_balance_maximal_sparsity})}" {yshift=5pt}] 
    &
        \parbox{2cm}{\centering\small Normalised CF-Freq \\ $\psivhat$}
    &
    \parbox{2.2cm}{\centering\small CF-Freq \\ \scriptsize\citet{Mothilal2021}\\ \small $\psiv$}
    \arrow[l, "\scriptsize\parbox{1.8cm}{\center \textbf{Normalisation} (Efficiency, \cref{sec:efficiency})}"]
    \\
    &
    &
    \parbox{2.25cm}{\centering\small {Dictators Symmetric}\\Solution Concepts \\ $\gammav$}
    \arrow[ur, Leftrightarrow, swap, "\scriptsize\parbox{2.75cm}{\center \textbf{Maximal Sparsity of Counterfactuals}\\(\cref{theo:extension-solution-concepts})}" {xshift=-10pt,yshift=-5pt}] 
    &
    &
\end{tikzcd}

\begin{tikzpicture}[overlay,remember picture]
\draw[rounded corners, color=niceblue] (-2.7, -0.25) rectangle (7.2, 4.5) {};
\node[color=niceblue] at (1.2, 0.05) {\footnotesize\parbox{3cm}{\raggedleft\textbf{This Paper}}};
\end{tikzpicture}
    \caption{Diagram showing the connection between game-theoretic feature attributions and counterfactual feature importance techniques. Nodes are techniques, edges ($\rightarrow$) show the change from one technique to another, and the double-sided edge ($\Leftrightarrow$) shows the equivalency relationship (with its conditions). See \cref{sec:shap_to_cfx,sec:connection,sec:discuss-game-theory} for more details on the journey.}
    \label{fig:connection}
\end{figure*}
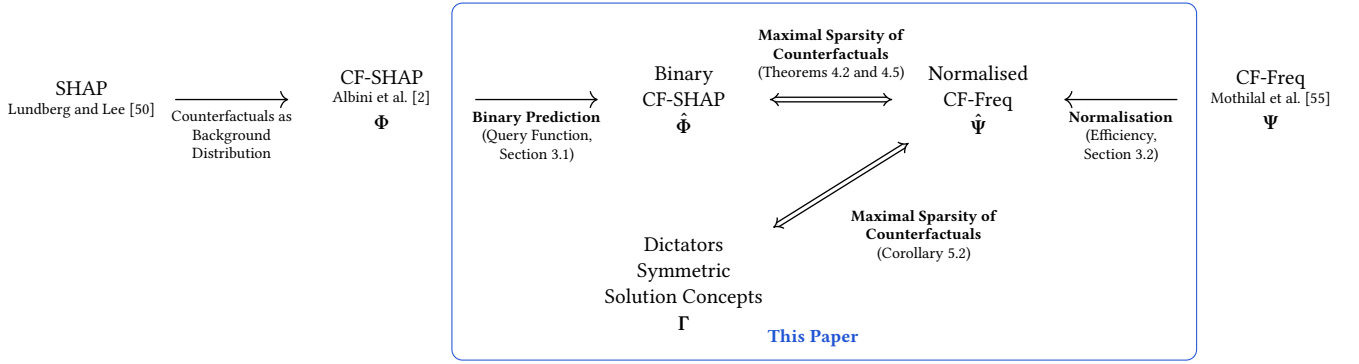

In the context of machine learning models the players are the features of the model and several ways have been proposed to simulate feature absence in the characteristic function, e.g., retraining the model without such feature \citep{Strumbelj2010}.
In particular, SHAP \citep{Lundberg2017} simulates the absence by marginalising over the marginal distributions of the features. In practice, the marginals are estimated as a uniform distribution over a finite number of points $\X$ called the \emph{background dataset} --- typically the training set (or a sample thereof). 
\EA{The formal definition of SHAP values follows.}

\begin{definition}[SHAP]
The \emph{SHAP values} for a query instance $\x$ with respect to a background dataset $\X$ are the \emph{Shapley values} of a game with the following characteristics function. \citep{Lundberg2017}
$$
    v(S) = \E_{\xp \sim \X}[ f\left(\tuple{\x_{S}, \xp_{\bar{S}}}\right)]
$$
where $\tuple{\x_{S}, \xp_{\bar{S}}}$ indicates a model input with feature values $\x$ for features in $S$ and $\xp$ for features not in $S$.
\end{definition}

The fact that SHAP simulates feature absence with a \emph{background dataset} means that it explains a prediction of an input \emph{in contrast} to a distribution of background points \citep{Merrick2020}.
Starting from this observation \citet{Albini2022} proposed {Counterfactual SHAP}: a variant of SHAP, using counterfactuals rather than the training set as the background dataset. This results in an explanation that can identify which features, if changed, would result in a different model decision better than SHAP.

\begin{definition}[CF-SHAP]\label{def:shap}
    Given a query instance $\x$, the \emph{Counterfactual SHAP values}, denoted with $\phiv$, are the \emph{SHAP values} with respect to $\Xp$ such that $\Xp$ is a set of counterfactuals for $\x$. \citep{Albini2022}
\label{cf-shap}
\end{definition}

\EA{We recall that the mathematical properties of Shapley values, and by extension SHAP values, of \emph{efficiency}, \emph{null-player} and \emph{strong monotonicity} also apply to CF-SHAP values \cite{Shapley1951,Lundberg2017,Albini2022}. In particular, in \cref{sec:efficiency} we will show the key role that the \emph{efficiency} property plays in drawing the connection with counterfactual explanations.}


\section{Incongruity of SHAP and Counterfactuals}\label{sec:shap_to_cfx}





There are three dimensions along which SHAP and counterfactual explanations differ:
\begin{enumerate}
\item The {\textbf{query function}} used to generate explanations. SHAP queries the model using $f$ to attribute to each of the features a part of the model \emph{output}; counterfactuals instead query the model using $F$ with the aim of finding a point with a different \emph{prediction}.
\item The \textbf{{efficiency}} of explanations. The game-theoretic property of efficiency that SHAP values requires them to add up to the model output. This is not inherently true for counterfactual explanations.

\item The \textbf{{granularity}} of the explanation. A single counterfactual does not inherently ``rank'' features based on their effect on the output of the model: it only shows \emph{which features} to modify to get a different prediction. On the other hand, SHAP \emph{assigns a score} to each feature based on their impact on the model output (even when using a single data point as background).
\end{enumerate}

In this section we present how we propose to carefully change these dimensions in order to draw an equivalency relationship between CF-SHAP and \textsc{CF-Freq}.
A summary diagram of this journey is in \cref{fig:connection}.

\EA{We remark, as mentioned in \cref{sec:introduction}, that} \textbf{the theory established in this paper applies to \emph{any} counterfactual explanation generation technique}. 
\EA{We also remark that --- while in this section we focus on the connection of counterfactual explanations with Shapley-values based explanation because of its popularity in the XAI field as well as in broader machine learning community --- \textbf{our theoretical results can be generalised to other game-theoretic solution concepts beyond Shapley values} (see \cref{sec:discuss-game-theory})}.



\subsection{Query Function}\label{sec:model_behaviour}


One key difference between SHAP and counterfactual generation engines is that they interact with the model differently. This is due to the different goals of the two explanations:
while counterfactual generation algorithms aim to find a point $\xp$ with a different \emph{model prediction}, SHAP goal is to attributes the \emph{model output} to the features.
This means that, concretely, when generating the explanations, these methods query the model using different functions: SHAP uses $f$ while counterfactual generation engines use $F$.

In order to bring the two explanations under the same paradigm, we propose to change the characteristics function of CF-SHAP (\cref{cf-shap}) to use $F$ (rather than $f$). 
\EA{We now formally define the resulting feature attribution.}

\begin{definition}[\textsc{Bin-CF-SHAP}]\label{def:binary-cf-shap}
  Given a query instance $\x$ the \emph{Binary CF-SHAP values}, denoted with $\phivhat$, are the \emph{SHAP values} of a game with the following \emph{characteristic function}.
    $$
    v(S) = \E_{\xp \sim \Xp}[ F\left(\tuple{\x_{S}, \xp_{\F \setminus S}}\right)]
    $$
    where $\Xp$ is a set of counterfactuals for $\x$.
\end{definition}

We note that CF-SHAP already made use of $F$, but only to compute the counterfactuals used as background dataset. 
\EA{Instead, with this change} to the characteristics function, CF-SHAP becomes completely ``insensitive'' to changes in model \emph{outputs} (probability) that do not also give rise to a change in the model \emph{prediction} (class). 

We remark that changing the \emph{characteristic function} of SHAP implies that the resulting attribution is still a vector of Shapley values and, as such, it retains all the (desirable) game-theoretic properties \citep{Sundararajan2020}.
In fact, it is not uncommon in the feature attribution literature to query the model using functions other than the model output $f$. \citet{Covert2020} analysed such \emph{query functions} --- in their work called \emph{model behaviours}. 
Nevertheless, we emphasise that querying the model using $F$, as we propose, is \emph{novel} to the feature attribution literature.


\subsection{Efficiency of Explanations}\label{sec:efficiency}

Shapley values satisfy the \emph{efficiency} property\footnote{\EA{The game-theoretic property of \emph{efficiency} \cite{Shapley1951} is sometimes referred to as \emph{additivity} \cite{Lundberg2017} in the XAI literature.}}, an essential part of many of their axiomatisations \citep{Shapley1951,Young1985}. In the context of SHAP values, it requires that:

$$
\sum_{i \in \F} \phi_i = f(\x) - \E_{\x' \in \X}[f(\x')] .
$$ 

In other words it requires the SHAP values to \emph{truly} be a \emph{feature attribution} distributing the model output among the features.
It can be trivially shown that in the case of \textsc{Bin-CF-SHAP} (\cref{def:binary-cf-shap}) the efficiency property simplifies to the following expression (see \cref{prop:efficiency} for more details).
$$
    \sum_{i \in \F} \phihat_i = 1
$$
We note that the efficiency property is not satisfied by \textsc{CF-Freq}.
Given that \textsc{CF-Freq} has not been defined with the goal to satisfy such game-theoretic property, \textsc{CF-Freq} explanations are not \emph{feature attributions}, i.e., they will not attribute to each of the features part of the model output. They instead sum up to a value that could greater or lesser than $1$ depending on the query instance.

In order to ensure that \textsc{CF-Freq} explanations satisfy the efficiency property, we propose to add a \textbf{normalisation term}.
We call the resulting feature importance \textsc{Norm-CF-Freq}. 
Concretely, while \textsc{CF-Freq} gives to each of the modified features in a counterfactual an importance of $1$, \textsc{Norm-CF-Freq} instead gives them an importance of $1/c$ where $c$ is the number of modified features in the counterfactual. If multiple counterfactuals are given for the same query instance, the (element-wise) average of such feature importance will be computed, similar to \textsc{CF-Freq}.

\begin{definition}[\textsc{Norm-CF-Freq}]\label{def:cf_hamming}\label{def:normalised_cf_freq}
    Given a query instance $\x$ and a set of counterfactuals $\Xp$, the \emph{Normalised \textsc{CF-Freq}} explanation, denoted with $\psivhat$, is defined as follows. 
    $$\psivhat = \E_{\xp \in \Xp}[ \frac{\I[\x \neq \xp]}{\norm{ \I[\x \neq \xp] }} ]$$
\end{definition}


We note that such modifications of a solution concept to enforce the efficiency property is not foreign in the game theory literature, e.g., in the case of normalised Banzhaf values \citep{vandenBrink1998}.

\subsection{Granularity of Explanations}\label{sec:feature_removal_strategy}




SHAP and counterfactual explanations differ in terms of the granularity of the explanations they can provide. 

On one hand, (CF-)SHAP \emph{assigns a score} to each feature based on their impact on the model output for each of the counterfactual example in its background distribution. 
On the other hand, a counterfactual \emph{does not} inherently provide any ``score'' describing the effect of each feature on the output of the model: it can only provide a binary assessment on the role of a feature in changing the prediction, i.e., ``is the feature modified in the counterfactual or not?''.


Consider the following toy example where \textsc{Bin-CF-SHAP} and \textsc{Norm-CF-Freq} explanations, denoted with $\phivhat$ and $\psivhat$ respectively, are generated using a single counterfactual.
$$
\begin{matrix}
\x = \begin{pmatrix} 1 & 1 & 1 & 1 & 1 & 1 \end{pmatrix}^T \\[2pt]
\xp = \begin{pmatrix} 0 & 0 & 0 & 0 & 0 & 1 \end{pmatrix}^T \\[4pt]
F(\x) = \Ind\bigl[x_1 \land \bigl(x_2 \lor (x_3 \land x_4) \bigr)\bigr] \\[5pt]
\phivhat = \begin{pmatrix} 7/12 & 3/12 & 1/12 & 1/12 & \hphantom{1}0\hphantom{/} & 0 \end{pmatrix}^T \\[3pt]
\psivhat = \begin{pmatrix} 1/5\hphantom{1} & 1/5\hphantom{1} & 1/5\hphantom{1} & 1/5\hphantom{1} & 1/5 & 0 \end{pmatrix}^T  
\end{matrix}
$$

We note how \textbf{\textsc{Norm-CF-Freq} gives equal importance to all the features} while \textbf{\textsc{Bin-CF-SHAP} is able to differentiate} the features in the counterfactual that are:
\begin{enumerate}
    \item \emph{necessary} for its validity ($x_1$): if the value of one of such features is replaced back with that in the query instance, the counterfactual is not valid anymore;
    \item only part of a \emph{sufficient} set ($x_2, x_3$ and $x_4$)
    for its validity: replacing the value of a sufficient feature with that in the query instance \emph{alone} will not invalidate the counterfactual, but it will when this is done in combination with the replacement of other features' values;
    \item \emph{spurious} ($x_5$): replacing their values back to that in the query instance will not invalidate the counterfactual under any circumstance. These 
    could be features that have been perturbed solely to increase the plausibility of the counterfactual despite having no impact on the model prediction or even not being used by the model at all. 
\end{enumerate}

\textsc{Bin-CF-SHAP} gives the largest attribution to features falling into (A), smaller attributions to those falling into (B) and zero attribution to those falling into (C).

\EA{That the inability of \textsc{CF-Freq} explanations to differentiate between \emph{necessary}, \emph{sufficient} and \emph{spurious} features represents a key limitation with respect to Shapley-values based feature attributions. 
This limitation will be made even more evident by the empirical results presented in \cref{sec:experiments}.
}


\section{Connecting SHAP and Counterfactuals}\label{sec:connection}

In \cref{sec:shap_to_cfx} we discussed what are the differences that exists between feature attributions and counterfactual explanations. In particular, this discussion resulted in two explanation techniques:
\begin{itemize}
    \item on the feature attributions side, \textsc{Bin-CF-SHAP} (\cref{def:binary-cf-shap}), a variant of CF-SHAP that queries the model using only the (binary) \emph{decision function};
    \item on the counterfactuals side, \textsc{Norm-CF-Freq} (\cref{def:normalised_cf_freq}), an \emph{efficient} variant of \textsc{CF-Freq}.
\end{itemize}

In this section we will present the main results of this paper, proving that --- after imposing some conditions on the counterfactuals --- \textsc{Bin-CF-SHAP} and \textsc{Norm-CF-Freq} are, in fact, the same explanation. 

We recall from \cref{sec:feature_removal_strategy} that \textsc{Bin-CF-SHAP} explanations provide a more fine grained account of the contributions of the features in counterfactuals to the output of the model than \textsc{Norm-CF-Freq} \EA{explanations}.
Therefore, towards finding an equivalency relationship between the two approaches, 
we must add additional constraints on counterfactuals such that \textsc{Bin-CF-SHAP} gives to all 
features in the counterfactual an equal attribution, similarly to what \textsc{Norm-CF-Freq} does.

We pose that this can be done by enforcing an additional property on counterfactuals called \textbf{\emph{maximal sparsity}}. 
Maximal sparsity requires a counterfactual to have the least number of changes (with respect to the query instance) for it to be valid or, in other words, it requires all the features in the counterfactual to be \emph{necessary}.

\begin{definition}[Maximal Sparsity]\label{def:maximally_sparse}
    A \textit{counterfactual} $\xp$ for a \textit{query instance} $\x$ is \emph{maximally sparse} iff:
    $$
        F(\xp) \neq F\round{\tuple{\x_{S}, \xp_{\bar{S}}}} \quad \forall S \subseteq C : S \neq \emptyset
    $$
    where $C$ is the set of features in $\xp$ that are different from those in $\x$, i.e., $C = \{ i \in \F : x_i \neq x_i' \}$.
\end{definition}

Note that it is \emph{always} possible to generate a maximally sparse counterfactual from \emph{any} counterfactual \EA{(\emph{independently} of the counterfactual generation technique)} by selecting a (proper or improper) subset of the features in the counterfactual. We denote the set of such subsets of $\F$ with $\MS(\F)$.
For example, in the running example, 2 such subsets exist:
$$
    \MS(\F) = \{ \{ 1, 2 \}, \{ 1, 3, 4 \} \} . 
$$

We now prove that maximal sparsity is indeed sufficient for the equivalency of \textsc{Bin-CF-SHAP} and \textsc{Norm-CF-Freq}.

\newcommand{\shapcfxconnectiontheoremtext}{
    Given a query instance $\x$, a set of counterfactuals $\Xp$, the \textsc{Bin-CF-SHAP} values $\phivhat$ and the \textsc{Norm-CF-Freq} explanation $\psivhat$ with respect to $\Xp$:
    $$
        \Xp \text{ are \textsc{maximally sparse}} \quad\Rightarrow\quad \phivhat =  \psivhat .
    $$
}
\begin{theorem}
    \label{theo:connection}
    \shapcfxconnectiontheoremtext
    \begin{proof}
        See \cref{appendix:proofs}. 
    \end{proof}
\end{theorem}
\newtheorem*{shapcfxconnectiontheorem}{\textbf{\textup{\cref{theo:connection}}}}

Maximal sparsity allows us to draw an equivalency relationship between the two explanation types. However, while maximal sparsity of counterfactuals is easy to define, it is, nonetheless, a strong requirement. An obvious question that arises is if there exists a weaker requirement allowing to draw the same equivalency relationship.
We now introduce few notions that allows us to describe a weaker, yet more complex requirement on counterfactuals, that allows us to draw the same equivalency relationship.


\begin{definition}[Weak Maximal Sparsity]\label{def:weak_maximal_sparsity}
    A \textit{counterfactual} $\xp$ for a \textit{query instance} $\x$ is \emph{weakly maximally sparse} iff $\forall i \in C$:
    $$
        \begin{matrix}
        \exists S \subseteq C \setminus \{i \} \,:\,\,
        F\round{\tuple{\x_{S\cup\{i\}}, \xp_{\bar{S} \setminus \{i\}}}} \neq F\round{\xp}\\[6pt]
        \text{where } C = \{ i \in \F : x_i \neq x_i' \} .
        \end{matrix}
    $$
\end{definition}
Intuitively, \emph{weak maximal sparsity} requires that counterfactuals do not contain spurious features. We note that it is \emph{always} possible to generate a weakly maximally sparse counterfactual from any counterfactual \EA{(\emph{independently} of the counterfactual generation technique)} by selecting a (proper or improper) subset of the features in the counterfactual. We denote the set of such subsets with $\WMS(\F)$. For the running example, three such subsets exist:
$$
    \WMS(\F) = \{ \{1, 2\}, \{ 1, 3, 4\}, \{ 1, 2, 3, 4 \} \} .\\[2pt]
$$

\begin{definition}[Equal Maximal Sparsity]\label{def:equal_maximal_sparsity}
 A \textit{counterfactual} $\xp$ for a \textit{query instance} $\x$ is \emph{equally maximally sparse} iff:
 $$ 
 |C| = 1 \,\,\, \lor \,\,\, \forall i,j \in C ,
 \sum_{\substack{S \in \WMS(\F)\\ i \in S}} \frac{\xi(S)}{|S|} = 
 \sum_{\substack{S \in \WMS(\F)\\ j \in S}} \frac{\xi(S)}{|S|}\\[5pt]
 $$
 where $C = \{ i \in \F : x_i \neq x_i' \}$ and $\xi : 2^{\F} \rightarrow \R$ is:
$$
\xi(S) = 
\begin{cases}
    1                                   &  \text{ if } S \in \MS(\F)\\
    1 - \sum_{T \in \WMS(S)} \delta_T   & \text{ otherwise}
\end{cases}
$$

\end{definition}

We note that equal maximally sparsity is not a simple condition to enforce on counterfactuals. In fact, requiring a counterfactual to be equal maximally sparse is, in practice, equivalent to requiring that the \textsc{Bin-CF-SHAP} values with respect to such single counterfactual must all be equal.


\newcommand{\shapcfxconnectionbalancedtext}{
    Given a query instance $\x$, and a counterfactual $\xp$, the \textsc{Bin-CF-SHAP} values $\phivhat$ and the \textsc{Norm-CF-Freq} explanation $\psivhat$ with respect to $\xp$:
    $$
        \xp \text{ is \textsc{equally maximally sparse}} \quad \Leftrightarrow\quad \phivhat =  \psivhat .
    $$
}
\begin{theorem}
    \label{theo:connection_balance_maximal_sparsity}
    \shapcfxconnectionbalancedtext
    \begin{proof}
    See \cref{appendix:proofs}.
    \end{proof}
\end{theorem}
\newtheorem*{shapcfxconnectionbalanced}{\textbf{\textup{\cref{theo:connection_balance_maximal_sparsity}}}}

Since it can be proved that maximal sparsity implies equal maximal sparsity (see \cref{prop:maximal_to_equal}), then \cref{theo:connection_balance_maximal_sparsity_corr} follows \EA{from \cref{theo:connection,theo:connection_balance_maximal_sparsity} (see \cref{appendix:proofs} for more details)}.

\newcommand{\shapcfxconnectionbalancedcorrtext}{
    Given a query instance $\x$, a set of counterfactuals $\Xp$, the \textsc{Bin-CF-SHAP} values $\phivhat$ and the \textsc{Norm-CF-Freq} explanation $\psivhat$ with respect to $\Xp$:
    $$
        \Xp \text{ are \textsc{equally maximally sparse}} \quad \Rightarrow\quad \phivhat =  \psivhat .
    $$
}
\begin{corollary}
    \label{theo:connection_balance_maximal_sparsity_corr}
    \shapcfxconnectionbalancedcorrtext
\end{corollary}
\newtheorem*{shapcfxconnectionbalancedcorr}{\textbf{\textup{\cref{theo:connection_balance_maximal_sparsity_corr}}}}




We \EA{note} that, by enforcing counterfactuals to be sparse and potentially less plausible, we follow the ``true to the model'' paradigm \cite{Chen2020,Janzing2020} discussed in \cref{sec:background}, wherein the goal is to understand the model reasoning and not the causal relationship between the features.

\section{Connecting other game-theoretic solution concepts and  counterfactuals}\label{sec:discuss-game-theory}


In this section, we discuss the effects of querying the model with the binary decision function and using maximally sparse counterfactuals on game-theoretic interpretations of the explanation --- \EA{as proposed in \cref{sec:shap_to_cfx}} --- and how this allows us to extend \cref{sec:connection} \EA{equivalency} results to more \EA{game-theoretic} solution concepts \EA{beyond Shapley values}.

In particular, we will focus our analysis on the single-reference games in which the explanation game of SHAP can be decomposed \citep{Merrick2020}. The characteristic function of such games for \textsc{Bin-CF-SHAP} is defined as follows:
\begin{equation}\tag{$\triangle$}\label{eq:single_reference_characteristics}
    v_{\xp}(S) = F\left(\tuple{\x_{S}, \xp_{\F \setminus S}}\right) \quad
\end{equation}


\textbf{Voting games}.
The use of the \emph{binary} {decision function} $F$ rather than the \emph{continuous} function $f$ means that the resulting single-reference games are more specifically \emph{voting games}: games where the characteristics functiond are \emph{voting rules} describing the winning and losing coalitions of players (features).
$$
v \, : \, 2^{\F} \rightarrow \left\{ 0 \text{ (\textsc{lose})}\, ,\, 1 \text{ (\textsc{win})} \right\} .
$$
The winning coalitions of features are those preserving the query instance prediction $F(\x) = 1$, while the losing ones will give rise to a counterfactual.

The resulting \textsc{Bin-CF-SHAP}
values are, more specifically then, the average Shapley-Shubik power index \citep{Shapley1954} over single-reference games. 
Concretely, the Shapley-Shubik power index measures the fraction of possible voting \emph{sequences} in which a player (feature) casts the deciding vote, that is, the vote that first guarantees passage (same prediction) or failure (counterfactual).

\textbf{Unanimity Games}.
The enforcement of maximal sparsity on counterfactuals in the single-reference voting games of \textsc{Bin-CF-SHAP} means that the counterfactual is valid iff all the modified features are present. Such games, where a group of players (features) have veto power and together they exert common dictatorship, are known more specifically as \emph{unanimity games}.


\textbf{\EA{Generalisation} to solution concepts \EA{beyond Shapley values}}. 
Although the paper \EA{has so far focused on Shapley values because of its} popularity in the \EA{XAI and} machine learning communities, many other game-theoretic solution concepts exist.
The result of \cref{theo:connection} can be extended to any solution concept that equally distributes payoffs to the common dictators of \emph{unanimity games}.
\EA{We now formally define this property of a solution concept that we call \emph{dictators-symmetry}.}
\EA{
\begin{definition}
    A solution concept $\gammav$ is \emph{dictators-symmetric} if for any \emph{unanimity game} with common dictators $C$ it holds that:
    \begin{itemize}
    \item $\Gamma_i = 1 / |C| \,,\,\, \forall i \in C$, and
    \item $\Gamma_i = 0 \,,\,\, \forall i \in \notC$.
    \end{itemize}
\end{definition}
}

\EA{We now formally prove with \cref{theo:extension-solution-concepts} \emph{maximal sparsity} of the counterfactuals is indeed a sufficient condition for the equivalency of \textsc{Norm-CF-Freq} and any explanation that is a \emph{dictators-symmetric} solution concept.}


\begin{corollary}\label{theo:extension-solution-concepts}
    Given a query instance $\x$, a set of counterfactuals $\Xp$ and the \textsc{Norm-CF-Freq} explanation $\psivhat$ with respect to $\Xp$. If $\gammav$ is the average of \EA{\emph{dictators-symmetric} solution concepts of the single-reference games} then:
    \begin{equation*}  
        \Xp \text{ are \textsc{maximally sparse}}             
        \quad\Rightarrow\quad
        \gammav =  \psivhat.
    \end{equation*}
\end{corollary}
\begin{proof}
This trivially follows from \cref{theo:connection}. See \cref{appendix:proofs} for more details.
\end{proof}
Solution concepts to which \cref{theo:extension-solution-concepts} applies include:
\begin{itemize}
    \item 
    the \emph{Banzhaf value}, 
    from the homonym Banzhaf power index \citep{Penrose1946, Banzhaf1964, Coleman1968}, measuring the fraction of the possible voting \emph{combinations}
    in which a player casts the deciding vote;
    \item 
    the \emph{Deegan-Packel power index} \citep{Deegan1978, Deegan1983} that equally divides the power to the members of minimum winning coalitions;
    \item
    the \emph{Holler-Packel public good index} \citep{Holler1978, Holler1983} measuring the 
    fraction of minimum winning coalitions of which a player is a part.
\end{itemize}

This generalisation of our results to more game-theoretic solution concepts is especially \emph{important} in light of the criticisms raised to 
Shapley values 
in game theory \citep{Osborne1994} as well as in XAI \citep{Kumar2020}. In particular, the use of alternative solution concepts has been recently investigated, e.g., Banzhaf values \citep{Covert2020,Karczmarz2022} and it has been identified as a possible way to better align explanations with their applications' goals, e.g., feature selection \citep{Fryer2021a} or time series \citep{Villani2022}.

\section{Experiments}\label{sec:experiments}

In order to understand the effects of the changes to connect SHAP and counterfactuals presented in this paper, we run \EA{2 sets of experiments:}
\begin{enumerate}
\item \EA{We run an ablation study. We} measure the \textbf{explanations \EA{difference}} (for every step in \cref{fig:connection}).
\item We compute some popular \textbf{explanations metrics} that have been used to evaluate feature importance explanations in the literature.
\end{enumerate} 

We run experiments on three publicly available datasets widely used in the XAI literature: 
\textbf{HELOC} \citep{FICOCommunity2019} (Home Equity Line Of Credit), 
\textbf{Lending Club} \citep{LendingClub2019} and 
\textbf{Adult} \citep{Adult1994} (1994 U.S. Census Income).
For each dataset, we trained a (non-linear) \emph{XGBoost} model \citep{XGBoost}. We chose to train booting ensemble of tree-based models because, in the context of classification for tabular data, they are deemed as state-of-the-art in terms of performance \citep{Shwartz-Ziv2022}. 
However, we emphasise that the \EA{theoretical results of this paper} are model agnostic, i.e. they do \emph{not} depend on the type of model. 
We refer to Appendix~\ref{appendix:setup} for more details on the experimental setup. 

\begin{figure*}[!ht]
    \centering
    \includegraphics[width=.99\textwidth]{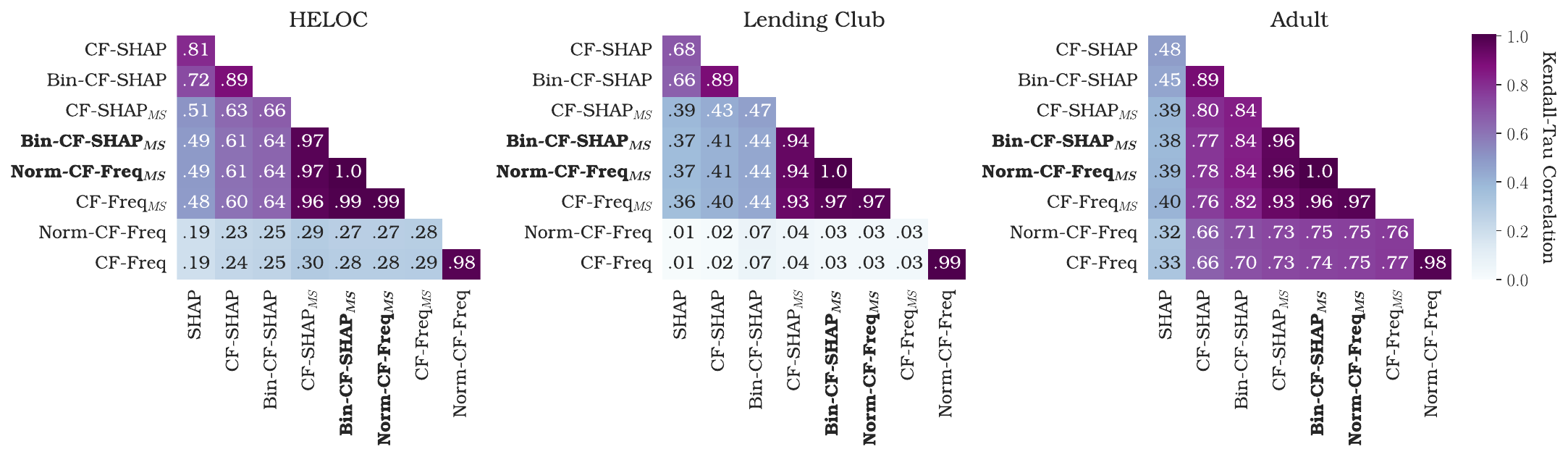}
    \caption{Average pair-wise Kendall-Tau Rank Correlation between explanations for different datasets. ($\MS$) Indicates explanations with maximally sparse counterfactuals.}
    \label{fig:difference}
\end{figure*}

We used TreeSHAP \citep{Lundberg2020}, KernelSHAP \citep{Lundberg2017} and CFSHAP \citep{Albini2022} and an in-house implementation of \textsc{CF-Freq} to generate the feature importance explanations. 
Similarly to \citet{Albini2022}, we used $K$-NN with $K=100$ and the Manhattan distance over the quantile space as distance metric to generate counterfactuals. 


We remark that the results presented in \EA{this paper (and in particular those in \cref{sec:connection,sec:discuss-game-theory}) hold \textbf{\emph{independently} of the algorithm that is used to generate (maximally sparse) counterfactuals}}. 
In fact, counterfactuals are only used to generate the background dataset for CF-Freq and SHAP-based feature importance.

As pointed out in \citet{Albini2022}, the choice of $K$-NN as the technique for the generation of counterfactuals \EA{in the context of the experiments} allows to analyse the resulting feature importance explanations performance separating it from the performance of the underlying counterfactual generation engine used to generate its background dataset.

To generate \emph{maximally sparse} counterfactuals we devised an exhaustive search algorithm that generates the closest maximally sparse counterfactual for each (non-maximally sparse) counterfactual passed to the feature importance explanations. We refer to \cref{appendix:maximal_sparse_cfx_method} for more details about the algorithm.

 \textbf{Explanations Difference}. \EA{In order} to draw a connection between \text{CF-SHAP} and \textsc{CF-Freq}, in \cref{sec:shap_to_cfx,sec:connection}, we presented three changes to the existing explanations: 
 \begin{enumerate}
\item querying the model using $F$; 
\item normalising \textsc{CF-Freq} explanation;
\item using maximally sparse counterfactuals.
 \end{enumerate}
 To understand the extent to which such changes impacted the explanations, we compute the pairwise Kendall-Tau rank correlation between the explanations for $1000$ examples.

\textit{Results -} \cref{fig:difference} show the results.
We note that:
\begin{enumerate}[label=\textbf{(\Alph*)}] 
\item While the normalisation only slightly effects explanations ($\tau = 0.97\text{-}0.99$), imposing maximal sparsity on the counterfactuals always has a significant effect on the resulting explanations ($\tau = 0.03\text{-}0.84$). This is consistent with results in the literature showing the large effects of the baseline on the explanations \citep{Sturmfels2020}.
\item The use of maximally sparse counterfactuals causes greater changes in the explanations in the frequentist family ($\tau = 0.03\text{-}0.77$) compared to SHAP-based explanations ($0.43\text{-}0.84$). This is coherent with what we presented in \cref{sec:feature_removal_strategy}: CF-SHAP already distinguishes between necessary, sufficient and spurious features but \textsc{CF-Freq} does not. Hence, the use of maximally sparse counterfactuals will have a greater effect on \textsc{CF-Freq}.
\item Querying the model with the binary prediction gives rise to more similar explanations when maximally sparse counterfactuals are used ($\tau = 0.94\text{-}0.97$) than otherwise ($\tau = 0.89$). This is expected: when using maximally sparse counterfactuals, the removal of any feature will invalidate the counterfactual, and therefore greatly reduce the model output (at least below the decision threshold). Such difference in output will tend to be closer (compared to using non-maximally sparse counterfactuals) to that obtained when using the binary prediction (i.e., always $1$).
\end{enumerate}

\begin{figure*}
    \includegraphics[width=.99\textwidth]{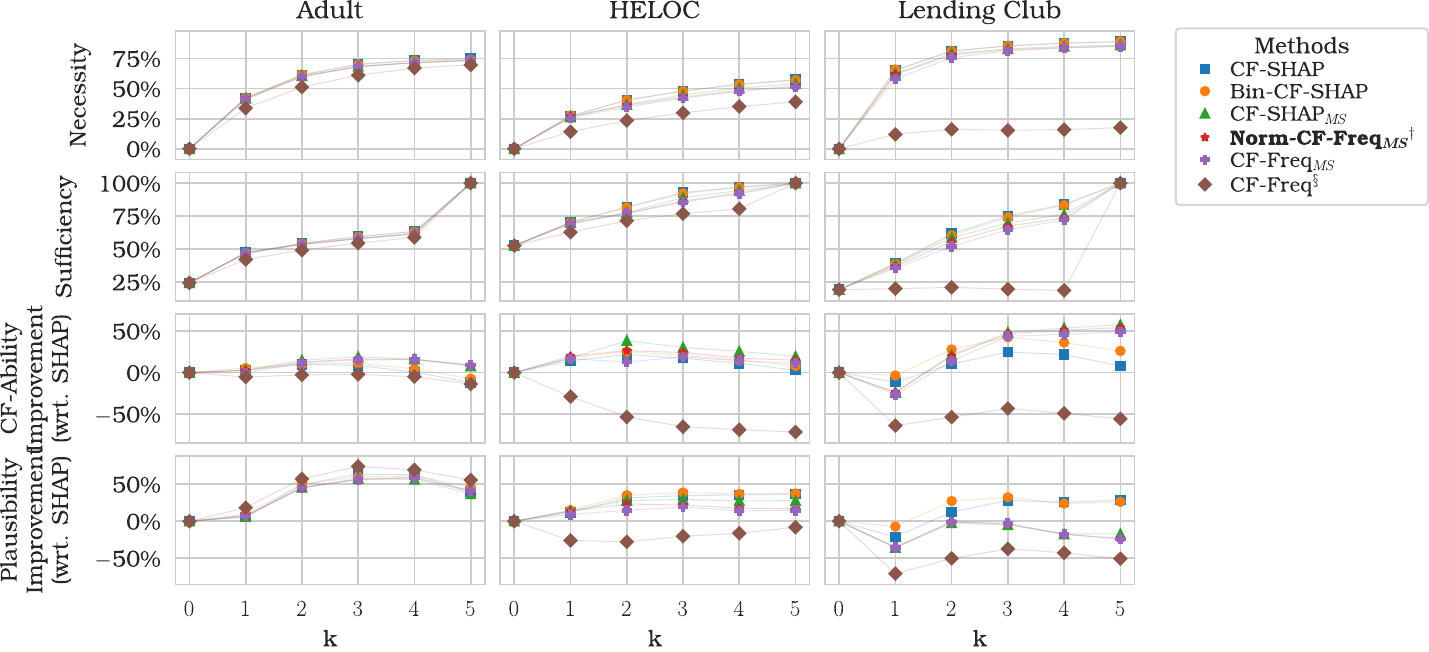}
    \caption{Metrics of explanations for different dataset (the higher the better). ($\MS$) Indicates explanations with maximally sparse counterfactuals; ($\dagger$) we report results only for \textsc{Norm-CF-Freq}$_{\MS}$ and not \textsc{Bin-CF-SHAP}$_{\MS}$ as they are equivalent (\cref{theo:connection});
    ($\mathsection$) we omit the results for \textsc{Norm-CF-Freq} as they are similar to those of \textsc{CF-Freq} (KS-test $D_{max} < 5\%$).
    }
    \label{fig:experiment_performance}
\end{figure*}

\textbf{Explanations Metrics}. To understand how the changes proposed in \cref{sec:shap_to_cfx,sec:connection} impact \textsc{CF-Freq} and \textsc{CF-SHAP} explanations we evaluated 4 metrics: \citep{Mothilal2021, Albini2022}.
\begin{itemize}
\item \emph{necessity} which measures the percentage of valid counterfactuals that can be generated when allowing only the top-$k$ features (according to a feature importance) to be modified;
\item \emph{sufficiency} which measures the percentage of \emph{invalid} counterfactuals that can be generated when allowing all the features but the top-$k$ to be modified;
\item \emph{counterfactual-ability} improvement which measures how often the proximity of counterfactuals induced by the explanations is better than that of SHAP. Counterfactuals are induced from the explanations by changing the top-$k$ features in the most promising direction (according to counterfactuals); The proximity is measured in terms of \emph{total quantile shift} \citep{Ustun2019};
\item \emph{plausibility} improvement which measures how often the plausibility of the same induced counterfactuals is better than that of SHAP. Concretely, the plausibility is measured as the density of the region in which they lie based on the distance from their $5$ nearest neighbours.
\end{itemize}

According to the framework of actual causality \citep{Halpern2016,Petsiuk2018} the assumption underpinning the definition of necessity and sufficiency is that the model output should change more when features with higher importance are modified (necessity) and it should change less when they are kept at their current value (sufficiency). The assumption behind counterfactual-ability and plausibility \cite{Albini2022} is that an explanation should suggest a way to plausibly change the decision with minimal cost (higher counterfactual-ability).


\textit{Results -} \cref{fig:experiment_performance} shows the results. We note that:
\begin{enumerate}[label=\textbf{(\Alph*)}] 
\item SHAP-based techniques do not have considerably different performance along any of the metrics. The most important difference among the explanations of this class is in terms of plausibility in the Lending Club dataset. This is not surprising: as discussed in \cref{sec:background,sec:connection} and in \citep{Sharma2022}, the enforcement of sparsity may give rise to less plausible explanations;
\item frequentist-based techniques, on the contrary, have substantially different performance. This is consistent with what we discussed in \cref{sec:feature_removal_strategy,sec:connection}: \textsc{CF-Freq} explanations are unable to discriminate between features in a counterfactual that are spurious, necessary for its validity or just part of a sufficient set to make it valid.
\end{enumerate}

\EA{We emphasise that}, if the goal of the explanation is to {determine what features are important} to change the prediction such that they are ``true to the model'' \citep{Chen2020}, these results further warn against using ``frequentist'' feature importance approaches (as \textsc{CF-Freq}) without a sparsity constraint as they cannot differentiate between modified features in the counterfactuals that are \emph{really used by the model}, and those that are not, as we highlighted in \cref{sec:feature_removal_strategy}.

\section{Conclusion and Future Work}

In this paper, we connected game theory-based feature attributions \EA{including (but not limited to) SHAP values} and ``frequentist'' approaches to counterfactual feature importance \EA{using the  fraction of counterfactuals that have a modified value}. 

In particular, we proved that by applying specific operations, they can be shown to be equivalent. We discussed the effect of such an equivalency theoretically, and then showed empirically the impact on explanations. \EA{This analysis highlighted the limitations of  ``frequentist'' approaches as feature importance technique 
and the important role of sparsity in counterfactual explanations}.

This paper provides avenues that could spur future research. 

Firstly, it would be interesting to analyse the connection between power indices using only minimum winning coalitions, e.g., Deegan-Packel's \citep{Deegan1978} and Holler-Packel's \citep{Holler1978}, and the property of maximal and weak maximal sparsity of counterfactuals proposed in this paper. More broadly, analysing if and how the game-theoretical interpretation of SHAP-based explanations aligns with the goal of the explanations would be of great interest.

Secondly, investigating how the results of this paper reflect on feature importance and counterfactual explanations that adopt a causal view of the world represents a future direction of great interest, e.g., \citep{Janzing2020,Galhotra2021,Aas2021,Frye2021,vonKugelgen2022,Russo2022}.

Lastly, while in this paper we limited our analysis of the connection between feature attributions and counterfactuals to the resulting feature importance explanations, it would be interesting to establish a more general connection between these two classes of approaches (e.g., between the SHAP values and distances between inputs and counterfactuals), as well as techniques falling under other XAI paradigms.

\begin{acks}
\textbf{Disclaimer}. 
This paper was prepared for informational purposes by
the Artificial Intelligence Research group of JPMorgan Chase \& Co. and its affiliates (``JP Morgan''),
and is not a product of the Research Department of JP Morgan.
JP Morgan makes no representation and warranty whatsoever and disclaims all liability,
for the completeness, accuracy or reliability of the information contained herein.
This document is not intended as investment research or investment advice, or a recommendation,
offer or solicitation for the purchase or sale of any security, financial instrument, financial product or service,
or to be used in any way for evaluating the merits of participating in any transaction,
and shall not constitute a solicitation under any jurisdiction or to any person,
if such solicitation under such jurisdiction or to such person would be unlawful.
\end{acks}

\bibliographystyle{ACM-Reference-Format}
\bibliography{main}

\appendix
\onecolumn
\section{Theoretical Results}\label{appendix:proofs}

In this appendix we report additional theoretical results together with the proofs of the results in \cref{sec:connection} that have been omitted from the main text for space and clarity of exposition reasons.

\subsection{Omitted Proofs}

We start by formally proving \cref{theo:connection,theo:connection_balance_maximal_sparsity}. We note that in the following proofs we will make use Shapley values axioms and properties as efficiency, null-player and symmetry. This are basic properties in the game theory literature, we refer the reader to Peters's game theory book \citep[][Chapter 17]{Peters2008} or Shapley's seminal work \citep{Shapley1951} for their formal definitions.

\begin{shapcfxconnectiontheorem}
\shapcfxconnectiontheoremtext
\begin{proof}
    Let's start by recalling that SHAP values calculation can be decomposed in the calculation of the SHAP values of single-reference games \citep{Merrick2020}:
    $$\phihat_i =  \E_{\xp \sim \Xp}[\phihat^{\xp}_i]
        \spaceornewline{\quad}
        \text{ where } \quad \phihat^{\xp}_i = \sum_{S \subseteq \F \setminus \{ i \}} w(S) \left[ v_{\xp}(S \cup \{i \}) - v_{\xp}(S) \right]
        \spaceornewline{\,\,,\,\,\,}
        v_{\xp}(S) = F\left(\tuple{\x_{S}, \xp_{\F \setminus S}}\right)$$

    We note that proving the thesis is equivalent to proving the following.
    \begin{equation}\tag{$\bigtriangleup$}\label{eq:proof:single_games}
        \phihat^{\xp}_i = \psivhat^{\xp}_i
        \quad \forall \xp \in \Xp \,,\, \forall i \in \F
    \end{equation}
    
    For each counterfactuals $\xp \in \Xp$, let's now consider the set of features that have been modified $C$ and its complement $\bar{C}$:
    $$
        {C} = \left\{ i \in \F : x'_i \neq x_i \right\} \quad\quad \bar{C} = \left\{ i \in \F : x'_i = x_i \right\}
    $$
    Proving \ref{eq:proof:single_games} is then equivalent to prove that $\forall \x' \in \Xp$:
    \begin{enumerate}[label=(\arabic*)]
    \item\label{eq:proof:zeros} $\forall i \in \bar{C}$ ,\, $\phihat^{\xp}_i = \psivhat^{\xp}_i = 0$;
    \item\label{eq:proof:ones1} $\forall i \in C$ ,\, $\phihat_i^{\xp} = \psivhat^{\xp}_i = 1$ if $|C| = 1$.
    \item\label{eq:proof:ones} $\forall i \in C$ ,\, $\phihat_i^{\xp} = \psivhat^{\xp}_i = 1 / \card{C}$ if $|C| > 1$.
    \end{enumerate}
    
    Let's start by first proving \ref{eq:proof:zeros}.
    It is trivial to observe that for all the features in $\bar{C}$ removing them has no effect because their value is equal in both the query instance and the counterfactual:
    $$
        \forall i \in \bar{C} \,,\,\, \forall S \subseteq \F \setminus \set{i} \,,\,\,\, v_{\xp}\left(S \cup \set{i}\right) - v_{\xp}\left(S\right) = 0 .
    $$
    Coincidentally, this is the definition of a \emph{null-player} (feature) therefore, by the null-player property of Shapley values, it follows that $\phiv^{\xp}_i = 0, \forall i \in \bar{C}$. This proves \ref{eq:proof:zeros}.

    In order to prove \ref{eq:proof:ones1}, let us now recall that by the \emph{efficiency} axiom of Shapley values, the attributions of the features must add up to the prediction of the model. Therefore --- since the features in $C$ are the only one with non-zero attribution --- it follows that:
    \begin{equation}\tag{A}\label{eq:proof:efficiency}
        \sum_{i \in C} \phihat_i^{\xp} = 1 \,.
    \end{equation}
    If $\card{C} = 1$ then \ref{eq:proof:ones1} trivially follows from (\ref{eq:proof:efficiency}).

    We will now prove \ref{eq:proof:ones}. Let us now recall that, since $\xp$ is maximally sparse by hypothesis, removing any of the features from the counterfactual will make it invalid:
    $$  
        \forall i \in C \,,\,\, \forall S \subset \F \setminus C
        \spaceornewline{\,\,,\,\,\,}
        v\bigl(S \cup C\bigr) = 1 \,\,\,\, \land \,\,\,\, v\bigl(S \cup (C \setminus \{i\})\bigr) = 0 \,.
    $$
    
    This implies that all the features in $C$ will have the same effect on the model prediction if removed:
    $$
    \forall i, j \in C : i \neq j\,, \,\,x_i' \neq \x_i\,, \,\,x_j' \neq \x_j \,, \spaceornewline{\quad}
    \forall S \subseteq \F \setminus \set{i, j} \,,\,\,\, v(S \cup \set{i}) = v(S \cup \set{j}) \,.
    $$
    This is the definition of symmetric players (features) therefore, by the \emph{symmetry} axiom of Shapley values, it follows that:
    \begin{equation}\tag{B}\label{eq:proof:symmetry}
        \phihat^{\xp}_i = \phihat^{\xp}_j, \forall i, j \in C \,.
    \end{equation}
    
    Then \ref{eq:proof:ones} follows from (\ref{eq:proof:symmetry}) and (\ref{eq:proof:efficiency}).

    The thesis then follows.
\end{proof}
\end{shapcfxconnectiontheorem}

\clearpage
\begin{shapcfxconnectionbalanced}
\shapcfxconnectionbalancedtext
\begin{proof}


    We note that proving the thesis is equivalent to proving the following.
    \begin{equation}\tag{$\bigtriangleup$}\label{eq:proof:single_games_b}
        \phihat_i = \psihat_i
        \quad \forall i \in \F
    \end{equation}
    
    Let's now consider the set of features that have been modified $C$ and its complement $\bar{C}$:
    $$
        {C} = \left\{ i \in \F : x'_i \neq x_i \right\} \quad\quad \bar{C} = \left\{ i \in \F : x'_i = x_i \right\}
    $$
    Proving \ref{eq:proof:single_games_b} is then equivalent to prove that $\forall \x' \in \Xp$:
    \begin{enumerate}[label=(\arabic*)]
    \item\label{eq:proof:zeros_b} $\forall i \in \bar{C}$ ,\, $\phihat^{\xp}_i = \psivhat^{\xp}_i = 0$;
    \item\label{eq:proof:ones_b1} $\forall i \in C$ ,\, $\phihat_i^{\xp} = \psivhat^{\xp}_i = 1$ if $|C| = 1$.
    \item\label{eq:proof:ones_b2} $\forall i \in C$ ,\, $\phihat_i^{\xp} = \psivhat^{\xp}_i = 1 / \card{C}$ if $|C| > 1$.
    \end{enumerate}
    We will now proceed by first proving \ref{eq:proof:zeros_b}.

    It is trivial to observe that for all the features in $\bar{C}$ removing them has no effect because their value is equal in both the query instance and the counterfactual:
    $$
        \forall i \in \bar{C} \,,\,\, \forall S \subseteq \F \setminus \set{i} \,,\,\,\, v_{\xp}\left(S \cup \set{i}\right) - v_{\xp}\left(S\right) = 0 .
    $$
    Coincidentally, this is the definition of a \emph{null-player} (feature) therefore, by the null-player property of Shapley values, it follows that $\phiv^{\xp}_i = 0, \forall i \in \bar{C}$. This proves \ref{eq:proof:zeros_b}.

    In order to prove \ref{eq:proof:ones_b1}, let us now recall that by the \emph{efficiency} axiom of Shapley values, the attributions of the features must add up to the prediction of the model. Therefore --- since the features in $C$ are the only one with non-zero attribution --- it follows that:
    \begin{equation}\tag{A}\label{eq:proof:efficiency_b}
        \sum_{i \in C} \phihat_i^{\xp} = 1 \,.
    \end{equation}
    then \ref{eq:proof:ones_b1} trivially follows from (\ref{eq:proof:efficiency_b}).
    
    In order to prove \ref{eq:proof:ones_b2}, let us now recall that Shapley values can be computed using the Harsanyi dividends \citep{Harsanyi1958}:
    \begin{equation}\tag{$\bigcirc$}\label{eq:harsanyi_shapley}
        \phihat_i = \sum_{\substack{S \in 2^{\F} \setminus \emptyset\\ i \in S}} \frac{\Delta_S}{|S|}
    \end{equation}
    where $\Delta_S$, called \emph{Harsanyi dividends}, are defined recursively as follows:
    \begin{equation}\tag{$\otimes$}\label{eq:harsanyi}
        \Delta_S = \begin{cases} 
            v(S) & \text{if } |S| = 1\\\ 
            v(S) - \sum_{T \subset S } \Delta_T & \text{otherwise}
        \end{cases}.
    \end{equation}
    
    Let us also recall that equal maximal sparsity requires all the counterfactuals $\xp$ to be such that:
    $$ 
     |C| = 1 \,\,\, \lor \,\,\, \forall i,j \in C ,
     \sum_{\substack{S \in \WMS(\F)\\ i \in S}} \frac{\xi(S)}{|S|} = 
     \sum_{\substack{S \in \WMS(\F)\\ j \in S}} \frac{\xi(S)}{|S|}\\[5pt]
     $$
     where $\xi : 2^{\F} \rightarrow \R$ is defined as follows:
    $$
    \xi(S) = 
    \begin{cases}
        1                                   &  \text{ if } S \in \MS(\F)\\
        1 - \sum_{T \in \WMS(S)} \delta_T   & \text{ otherwise}
    \end{cases}
    $$
    and where we recall that:
    \begin{itemize}
        \item $\MS(S)$ is the set of all the (proper or improper) subsets of $S$ that give rise to a maximally sparse counterfactual:
        $$\MS(S) = \{ T \subseteq S : \tuple{ \x_{\bar{T}},\xp_{T} } \text{ is maximally sparse} \} \, \text{;}$$
        
        \item $\WMS(S)$ is the set of all the (proper or improper) subsets of $S$ that give rise to a weak maximally sparse counterfactual:
        $$\WMS(S) = \{ T \subseteq S : \tuple{ \x_{\bar{T}},\xp_{T} } \text{ is weakly maximally sparse} \} \, \text{.}$$
    \end{itemize}

    We note that if we prove that:
    \begin{enumerate}[label=C.\Roman*]
    \item \label{eq:proof:harsones} $\Delta_S = 1 \quad \forall S \in \MS(\F)$
    \item \label{eq:proof:harszeros} $\Delta_S = 0 \quad \forall S \in \F \setminus \WMS(\F)$
    \end{enumerate}
    then, by \cref{def:equal_maximal_sparsity} and the definition of Shapley values with the Harsanyi dividends (\ref{eq:harsanyi_shapley}), \ref{eq:proof:ones_b2} follows and, in turn, the thesis follows.

    Let's then prove (\ref{eq:proof:harsones}). If $S \in \MS(\F)$, by the definition of the set $\MS(\F)$, it holds that $\tuple{ \x_{\F \setminus S}, \xp_{S} }$ is a maximally sparse counterfactual. Therefore, it trivially follows, from \cref{def:maximally_sparse}, that:
    $$
        v(S) = 1 \quad \text{and} \quad \forall T \subset S \,,\,\, v(T) = 0
    $$
    Then, by the Harsanyi dividends definition (\ref{eq:harsanyi}), (\ref{eq:proof:harsones}) follows.

    We now prove \ref{eq:proof:harszeros}. If $S \in \F \setminus \WMS(\F)$ and $S \subset T : T \in \MS(\F)$, then it trivially follows that $\Delta_S = 0$.
    
    If that is not the case, then $S \supset T : T \in MS(S)$ and it must contain at least a feature $i \in S$ such that $i$ is spurious, or in more formally such that:
    $$
        \forall i \in \bar{C} \,,\,\, \forall S \subseteq \F \setminus \set{i} \,,\,\,\, v_{\xp}\left(S \cup \set{i}\right) - v_{\xp}\left(S\right) = 0 .
    $$ 
    Note that this is the definition of a \emph{null-player}. Therefore, by Remark 4 in \citet{Dehez2017} --- stating that ``a player is null iff the dividends associated to coalitions containing that player are all equal to zero.'' --- \ref{eq:proof:harszeros} follows.

    Then the thesis follows.
\end{proof}
\end{shapcfxconnectionbalanced}

\begin{shapcfxconnectionbalancedcorr}
    \shapcfxconnectionbalancedcorrtext
    \begin{proof}
    The corollary follows trivially from \cref{theo:connection_balance_maximal_sparsity}.
    \end{proof}
\end{shapcfxconnectionbalancedcorr}

\subsection{Sparsity}

As mentioned in \cref{sec:connection}, the different notions of sparsity of counterfactuals defined in this paper are theoretically connected between each others. In particular, maximal sparsity implies equal maximal sparsity that, in turn, implies weak maximal sparsity. We now formally prove such relationships between these three properties of counterfactuals that we defined in \cref{sec:connection}.

\begin{proposition}[Maximal Sparsity $\Rightarrow$ Weak Maximal Sparsity]\label{prop:maximal_to_weak}
    If $\xp$ is maximally sparse then $\xp$ is also weakly maximally sparse. 
    And more in general, if for any $T \subseteq \F$, if $S \in \MS(T)$ then $S \in \WMS(T)$.
\end{proposition}
\begin{proof}
    The result follows trivially from \cref{def:weak_maximal_sparsity}.
\end{proof}

\begin{proposition}[Maximal Sparsity $\Rightarrow$ Equal Maximal Sparsity]\label{prop:maximal_to_equal}
    If $\xp$ is maximally sparse then $\xp$ is equal maximally sparse. 
\end{proposition}
\begin{proof}
    By \cref{def:maximally_sparse,def:weak_maximal_sparsity}, it follows that if $\x$ is maximally sparse than $\MS(\F) = \WMS(F)$.

    Therefore the following holds $\forall i \in C$ where $C = \{ i \in \F : x_i \neq x_i' \}$:
    $$
         \sum_{\substack{S \in \WMS(\F)\\ i \in S}} \frac{\xi(S)}{|S|} = \sum_{\substack{S \in \MS(\F)\\ i \in S}} \frac{\xi(S)}{|S|}
    $$
    Also, by \cref{def:equal_maximal_sparsity}, we can substitute $\xi(S)$, therefore:
    $$
         \sum_{\substack{S \in \WMS(\F)\\ i \in S}} \frac{\xi(S)}{|S|} = \sum_{\substack{S \in \MS(\F)\\ i \in S}} \frac{1}{|S|}
    $$
    which is a constant, thus the thesis follows.
\end{proof}

\begin{proposition}[Equal Maximal Sparsity $\Rightarrow$ Weak Maximal Sparsity]\label{prop:equal_to_weak}
    If $\xp$ is equal maximally sparse then $\xp$ is weak maximally sparse. 
\end{proposition}
\begin{proof}
    Let's consider $C = \{ i \in \F : \x_i \neq x_i' \}$. If $|C| = 1$ then the thesis follows trivially. 
    If instead $|C| > 1$ and $\xp$ is equally maximally sparse, by \cref{theo:connection_balance_maximal_sparsity}, it follows that:
    $$\phihat_i = \frac{1}{|C|}$$
    
    Now, if we assume, ad absurdum, $\xp$ is not weakly maximally sparse, then it means that $C$ contains at least a spurious features which gets a non-zero feature attribution.
    This is absurd given that $\phihat_i$ is a Shapley value and thus satisfy the \emph{null-property} of Shapley values, by which a null-player (spurious feature) always get zero attribution.
\end{proof}

\subsection{Additional Results}

In \cref{sec:efficiency} we mentioned how the efficiency property \citep{Peters2008} can be reduced to a simpler form for \textsc{Bin-CF-SHAP} feature attributions. We now formally prove such result in \cref{prop:efficiency}.

\begin{proposition}\label{prop:efficiency}
    In the context of \textsc{Bin-CF-SHAP}, the efficiency property of Shapley values simplifies to the following expression.
    $$
    \sum_{i \in \F} \phihat_i = 1
    $$
\end{proposition}
\begin{proof}
Let's recall that the efficiency property of Shapley values requires: \citep{Shapley1951}
$$
    \sum_{i \in \F} \phi_i= v(\x) - v(\emptyset)
$$
where $v$ is the characteristics function of the game for which we are computing Shapley values.

In particular, in the context of \textsc{Bin-CF-SHAP} values for which the characteristics function is defined as follows (see \cref{def:binary-cf-shap}):
$$
    v(S) = \E_{\xp \sim \Xp}[ F\left(\tuple{\x_{S}, \xp_{\F \setminus S}}\right)]
$$
the efficiency property simplifies to the following expression:
$$
    \sum_{i \in \F} = F(\x) - \E_{\x' \in \Xp}[F(\x')] .
$$
Since all $\x$ is the query instance, and $\xp \in \Xp$ are counterfactuals, by \cref{def:binary-cf-shap}, then it follows that:
$$
     \sum_{i \in \F} = 1 - 0  = 1.
$$
which proves the thesis.
\end{proof}

\clearpage
\section{Experimental Setup}\label{appendix:setup}

\begin{table*}[ht!]
\begin{center}
    \begin{small}
        \begin{tabular}{rcccccccc}
            \toprule
                \multirow{2}{*}{Dataset} &
                \multicolumn{3}{c}{Size} &
                Decision &
                \multicolumn{3}{c}{Model Performance$^{\dagger}$}\\
                & Features\! & \!Train Set\! & \!Test Set & Threshold$^{*}$ & ROC-AUC\! & \!Recall\! & \!Accuracy$^{\ddagger}$\! \\
            \midrule
                \!\textbf{HELOC} (Home Equity Line Of Credit)  & 23 & 6,909 & 2,962        & 0.4993 & 80.1\% & 72.8\% & 73.0\% \\
                \!\textbf{Lending Club} (Loan Data)  & 20 & 961,326 & 411,998    & 0.6367 & 69.6\% & 62.5\% & 72.17\% \\ 
                \!\textbf{Adult} (1994 US Census Income)       & 12 & 22,792 & 9,769        & 0.6811 & 92.3\% & 72.9\% & 86.6\% \\
            \bottomrule
        \end{tabular}
    \end{small}
\end{center}
    \protect\caption{
    Characteristics of the datasets and models used in the experiments.
    ($*$) The decision threshold is reported here in probability space (i.e., after passing the model output through a sigmoid);
    ($\dagger$) performance metrics are computed on the test set.
    }
    \label{table:datasets}
\end{table*}

\subsection{Datasets and Models}
To run the experiments we used 3 publicly available datasets. Table~\ref{table:datasets} describes in details the datasets.

We split the data using a stratified $70/30$ \emph{random} train/test split for HELOC and Adult. For Lending Club we split the data using a non-random $70/30$ train/test split based on the loan issuance date (available in the original data).

We trained an XGBoost model \citep{XGBoost} for each dataset. In particular, we run an hyperparameters search using Bayesian optimization using hyperopt~\citep{Hyperopt} for $1000$ iterations maximizing the average validation ROC-AUC under a 5-fold cross validation. To reduce model over-parameterization during the hyper-parameters optimization we penalized high model variance, i.e., for each cross-validation fold, instead of using $AUC_{val}$, we used $AUC_{val} + (AUC_{val} - AUC_{train})$ where $AUC_{train}$ and $AUC_{val}$ are the training and validation ROC-AUC, respectively.

To compute the decision threshold ($t$) we used a value such that the rate of positive prediction of the model (on the training set samples) was the same as the true rate of positive predictions (on the same samples).
Table~\ref{table:datasets} shows the decision threshold and the performance of each model.

\subsection{Feature Importance}
We used the following implementation in order to compute explanations.
\begin{itemize}
    \item \textbf{\textsc{SHAP}}. 
    We used the TreeSHAP implementation \citep{Lundberg2020} available through the \texttt{TreeExplainer} class in the \texttt{shap} package\footnote{The \texttt{shap} package can be found at \url{https://github.com/slundberg/shap}} (for Python).
    \item \textbf{\textsc{CF-SHAP}}. 
    We used CF-SHAP \citep{Albini2022} available through the \texttt{CFExplainer} class in the \texttt{cfshap} package\footnote{The \texttt{cfshap} package can be found at \url{https://github.com/jpmorganchase/cf-shap}} (for Python).
    \item \textbf{\textsc{Binary CF-SHAP}}. We used CF-SHAP in combinations with the KernelSHAP \citep{Lundberg2017} implementation available through the \texttt{KernelExplainer} class in the \texttt{shap} package. We used $10,000$ kernel samples to generate the KernelSHAP approximation. We note that, since we used KernelSHAP, the resulting {\textsc{Binary CF-SHAP}} explanations that we generated are an approximation of the exact Shapley values.
    \item \textbf{\textsc{CF-Freq}}. Given its simplicity, we implemented from scratch the explanation logic following the explanation definition \citep{Sharma2020,Mothilal2021}.
    \item \textbf{\textsc{Normalised CF-Freq}}. We implemented from scratch the explanation logic similarly to \textsc{CF-Freq}.    
\end{itemize}
We also remark, as mentioned in \cref{sec:background}, that we used the ``true-to-the-model'' interventional (a.k.a., non-conditional) version of SHAP (default setting of \texttt{shap} and \texttt{cfshap}).

\begin{algorithm*}[!th]
   \caption{Depth-first search-based algorithm to induce a maximally sparse counterfactual from any counterfactual}
   \label{alg:maxsparse}
\begin{algorithmic}
    \State \textbf{\textsc{MaxSparse}}($\x$, $\xp$, $\F$, $F$)
    \State {\bfseries Input:} query instance $\x$, counterfactual $\xp$, set of all features $\F$, $F$ model decision function
    \State $C = \{ i \in \F : x_i \neq x_i' \}$ $\quad\quad \triangleright$ Let's isolate the features that have been modified.
    \State $Fail = \{\} \quad\quad \triangleright$ Let's create a set for failed trials.
    \State $Succ = \{ C \} \quad\quad \triangleright$ Let's create a set for the successful trials.
    \State \textsc{MaxSparseRecurse}($\x$, $\xp$, null, $C$, $Succ$, $Fail$, $F$) $\quad\quad \triangleright$ Let's run the search recursively.

    \State $\triangleright$ We now select the maximally sparse counterfactual with minimum cost
    \State $\xp = $ null
    \State $c' = \infty$
    \For{$\bm{x}'' \in Succ$}
        \State $c'' =$ \textsc{cost}($\x$, $\bm{x}''$) $\quad\quad\triangleright$ We compute the cost/proximity of the counterfactual
        \If{$c'' < c'$ \textbf{or} $\xp$ \textbf{is} null}
            \State $\xp = \bm{x}''$
            \State $c' = c''$
        \EndIf
    \EndFor
    \State \textbf{Return} $\xp$
    
    \hfill\hfill
    
    \State \textbf{\textsc{MaxSparseRecurse}}($\x$, $\xp$, $\xp^P$, $C$, $Succ$, $Fail$, $F$)
    \State {\bfseries Input:} query instance $\x$, counterfactual $\xp$, parent of the counterfactual $\xp^P$, set of features $C$, successful trials $Succ$, failed trials $Fail$
    \If{$F(\xp) \neq F(\x)$}
        \State $\triangleright$ We remove the parent and add the current $\xp$ to the successful trials.
        \If{$\xp^P \in Succ$}
            \State $Succ.remove(\xp^P)$
        \EndIf
        \State $Succ.add(\xp)$
        \State $\triangleright$ We now expand the search recursively by removing one more feature from the counterfactual.
        \For{$i \in C$}
            \State $\C' = C \setminus \{ i \}$
            \State $\bm{x''} = \xp$
            \State $x_i'' = x_i$
            \If{$\bm{x''} \notin Fail \land \bm{x''} \notin Succ$}
                \State \textsc{MaxSparseRecurse}($\x$, $\bm{x}''$, $\xp$, $C'$, $Succ$, $Fail$, $F$)
            \EndIf
        \EndFor
    \Else
        \State $Fail.add(C)$
    \EndIf
\end{algorithmic}
\end{algorithm*}

\subsection{Counterfactuals}
To compute the $K$-nearest neighbours we used the implementation available in \texttt{sklearn.neighbours}. To make our results indifferent to the size of the dataset we limited the $k$-nearest neighbours to be selected among a random sample of $10,000$ samples from the training set.

\subsection{Maximally Sparse Counterfactuals}\label{appendix:maximal_sparse_cfx_method}

In \cref{sec:connection} we mentioned how it is always possible to generate a maximally sparse counterfactual from any counterfactual by selecting a subset of the features $C = \{i \in \F : x_i \neq \x_i' \}$ that have been modified in the counterfactual. 
In particular, as described in \cref{sec:experiments}, to run our experiments we devised an exhaustive search based algorithm. The pseudo-code for such algorithm is in \cref{alg:maxsparse}.
At a high-level \cref{alg:maxsparse} computes a maximally sparse counterfactual from any counterfactual as follows:
\begin{itemize}
    \item It explores all the possible subset of features $T \subset C$ using depth-first search (implemented through recursion).
    \item It prunes the search when it encounter a subset $\widehat{T}$ of features that does not give rise to a counterfactual; in this way it avoids to search any subset $Q \subset \widehat{T}$.
    \item After having generated $\MS(\F)$ --- the set of all the maximally sparse counterfactual that can be induced from $\xp$ --- it computes their cost based on a cost function provided by the user, and returns the counterfactual with the minimum cost.
    \item In our experiments, as mentioned in \cref{sec:experiments}, we used the \emph{total quantile shift} \citep{Ustun2019} as cost function for counterfactuals.
\end{itemize}

\subsection{Technical setup}
The experiments were run using a \texttt{c6i.32xlarge} AWS virtual machine with 128 vCPUs (64 cores of 3.5 GHz 3rd generation Intel Xeon Scalable processor) and 256GB of RAM. XGBoost parameter \texttt{nthread} was set to \texttt{15}.
We used a Linux machine running \texttt{Ubuntu 20.04}. We used \texttt{Python 3.8.13}, \texttt{shap 0.39.0}, \texttt{cfshap 0.0.2}, \texttt{sklearn 1.1.1} and \texttt{xgboost 1.5.1}.

\subsection{Source Code}
The source code to reproduce the experimental results in the paper will be made available at \url{https://www.emanuelealbini.com/cf-vs-shap-aies23}.

\clearpage
\section{Experimental Results}

\subsection{Explanations Difference}\label{appendix:additional_difference}

In \cref{sec:experiments} we showed how the explanations generated using different techniques differ in terms of their average pairwise Kendall-Tau rank correlations \citep{kendall1990correlation}. \cref{fig:additional_difference_spearman,fig:additional_difference_feature_agreement,fig:additional_difference_rank_agreement,fig:additional_difference_rbo} show the same results for additional metrics commonly used in the literature to measure the difference between the rankings that feature importance explanations provide. 
In particular, we show the results for Feature Agreement \citep{Ghorbani2019a}, Rank Agreement \citep{Krishna2022a}, Spearman Rank Correlation \citep{Spearman1987} and Rank Biased Overlap \citep{Webber2010, Sarica2022}.

\textit{Results} - We note that:
\begin{itemize}
\item In general, the results are consistent with the results in terms of Kendall-Tau correlation presented in \cref{fig:difference} in the main text.
\item The results in terms {feature and rank agreement} suggest that explanations tend to be more similar in their the top-3 features by importance than in their top-10. This is consistent with the literature \citep{Krishna2022a} that shows how different XAI techniques tend to agree more on the most important features when compared to those that are ranked as least important.
\end{itemize}

\begin{figure*}[!ht]
    \centering
     \begin{subfigure}[b]{\textwidth}
         \centering
         \includegraphics[width=.85\textwidth]{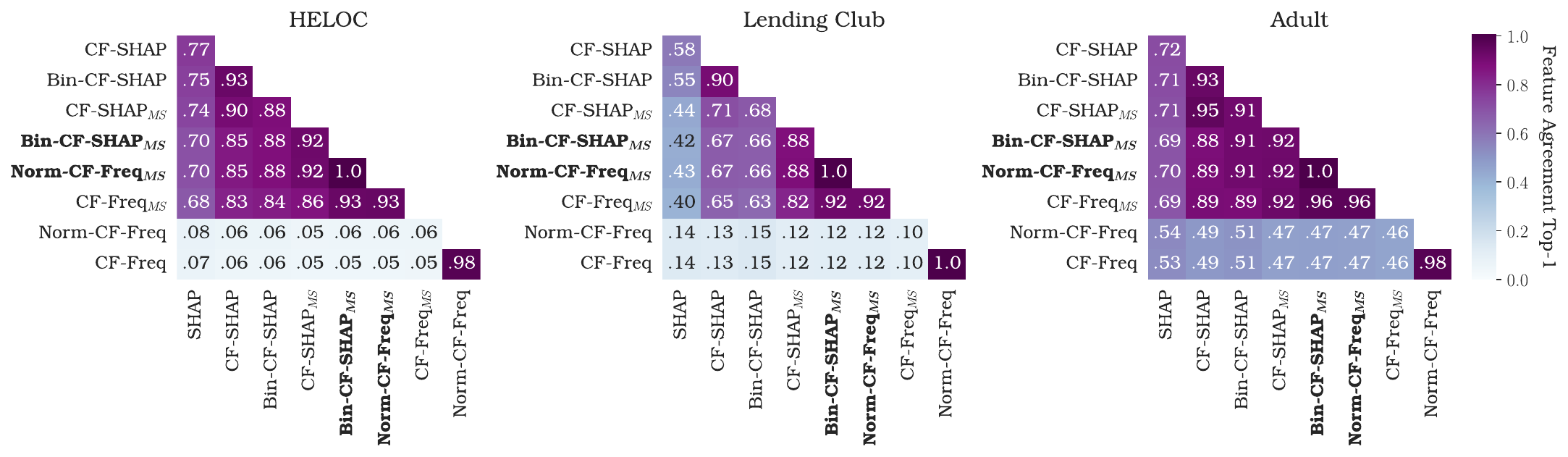}
         \caption{Top-1}
     \end{subfigure}
     \hfill
     \begin{subfigure}[b]{\textwidth}
         \centering
         \includegraphics[width=.85\textwidth]{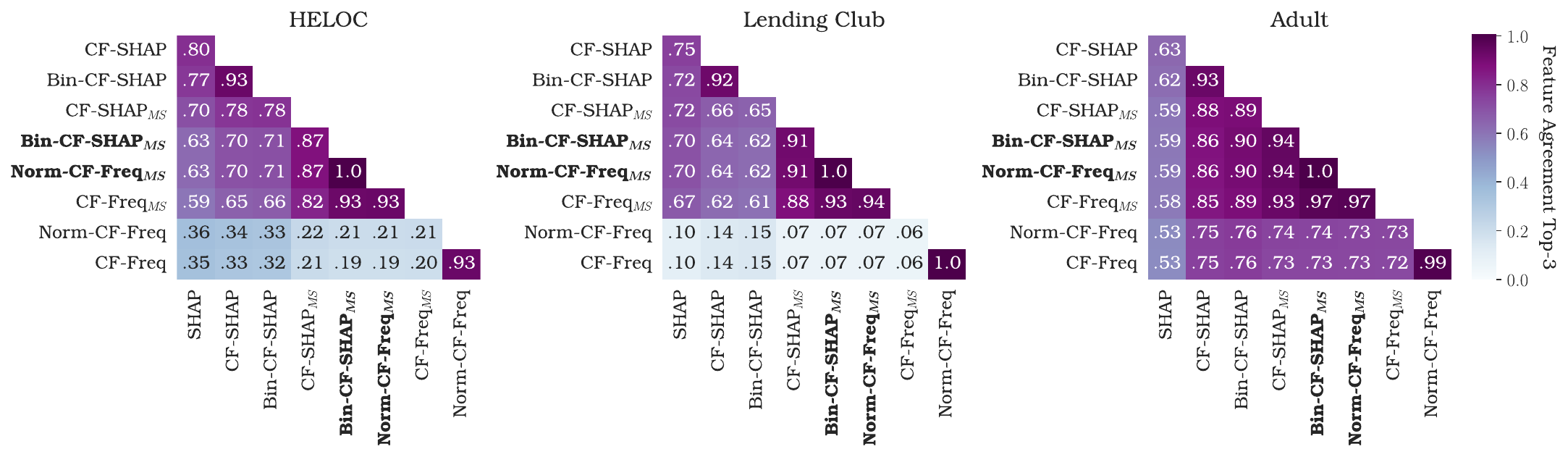}
         \caption{Top-3}
     \end{subfigure}
     \hfill
     \begin{subfigure}[b]{\textwidth}
         \centering
         \includegraphics[width=.85\textwidth]{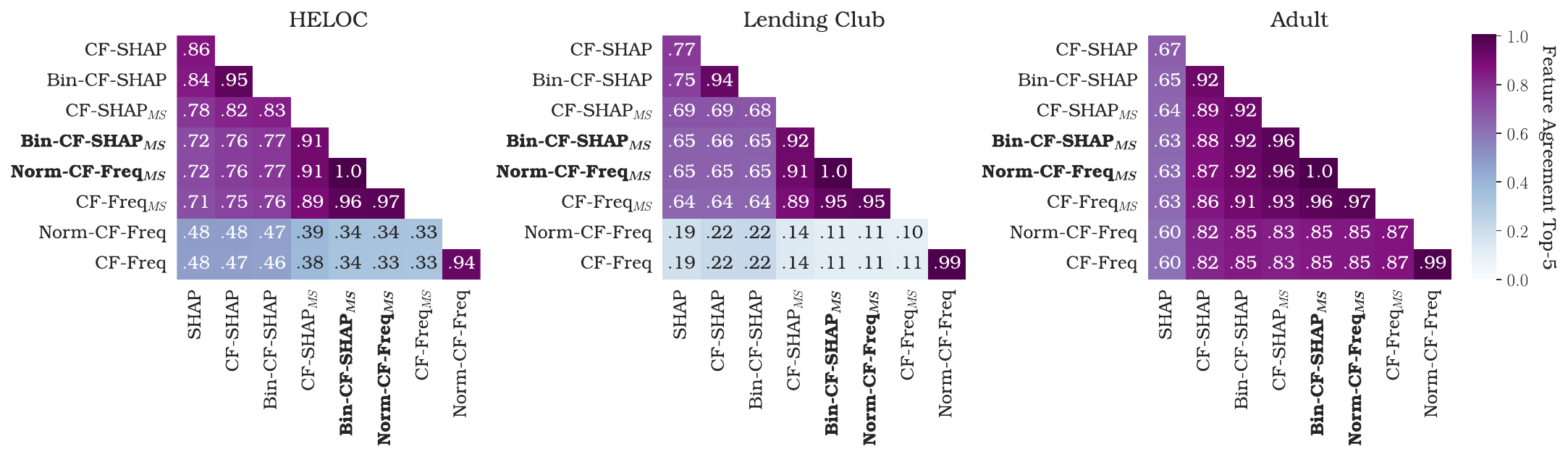}
         \caption{Top-5}
     \end{subfigure}
     \begin{subfigure}[b]{\textwidth}
         \centering
         \includegraphics[width=.85\textwidth]{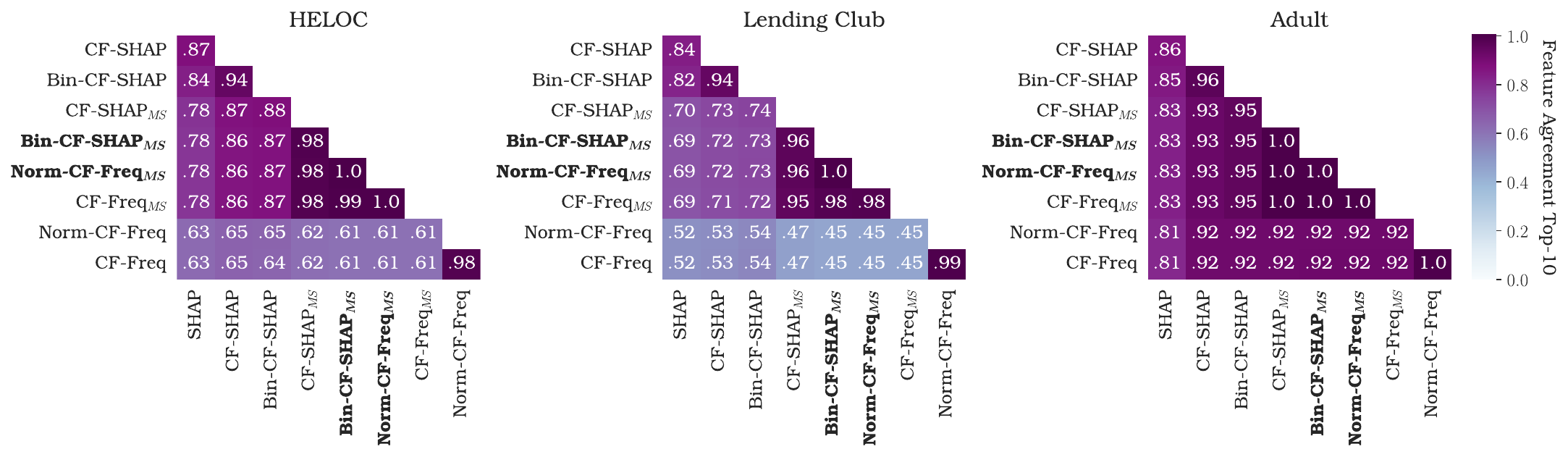}
         \caption{Top-10}
     \end{subfigure}
    \caption{Average pairwise Feature Agreement between explanations for different datasets. See \cref{appendix:additional_difference} for more details.}
    \label{fig:additional_difference_feature_agreement}
\end{figure*}

\begin{figure*}[!ht]
    \centering
     \begin{subfigure}[b]{\textwidth}
         \centering
         \includegraphics[width=.99\textwidth]{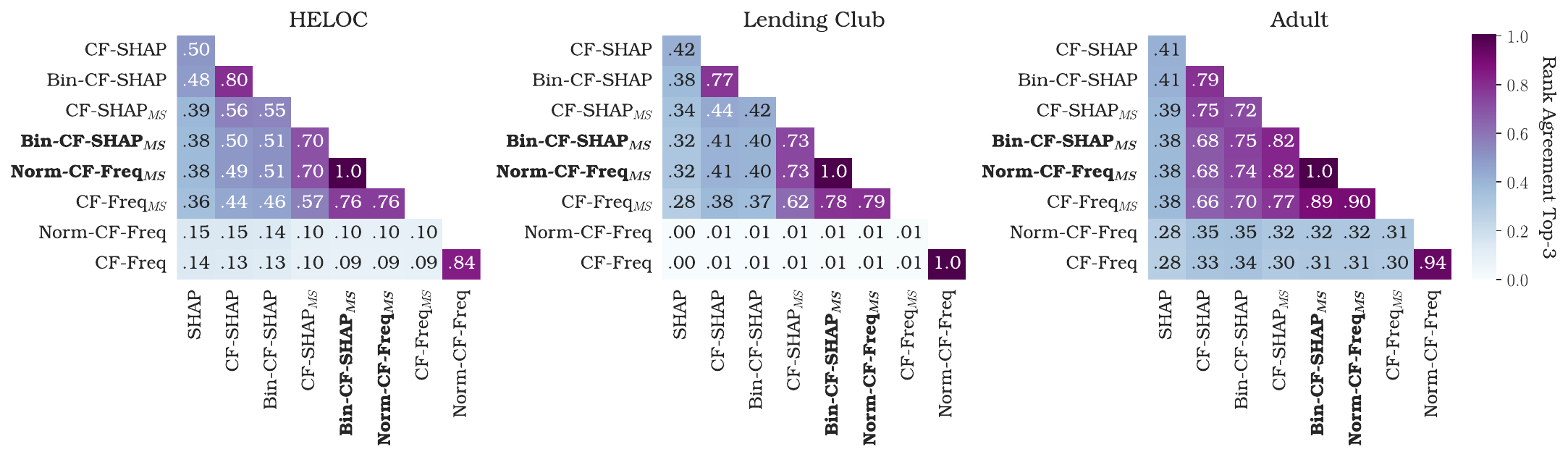}
         \caption{Top-3}
     \end{subfigure}
     \hfill
     \begin{subfigure}[b]{\textwidth}
         \centering
         \includegraphics[width=.99\textwidth]{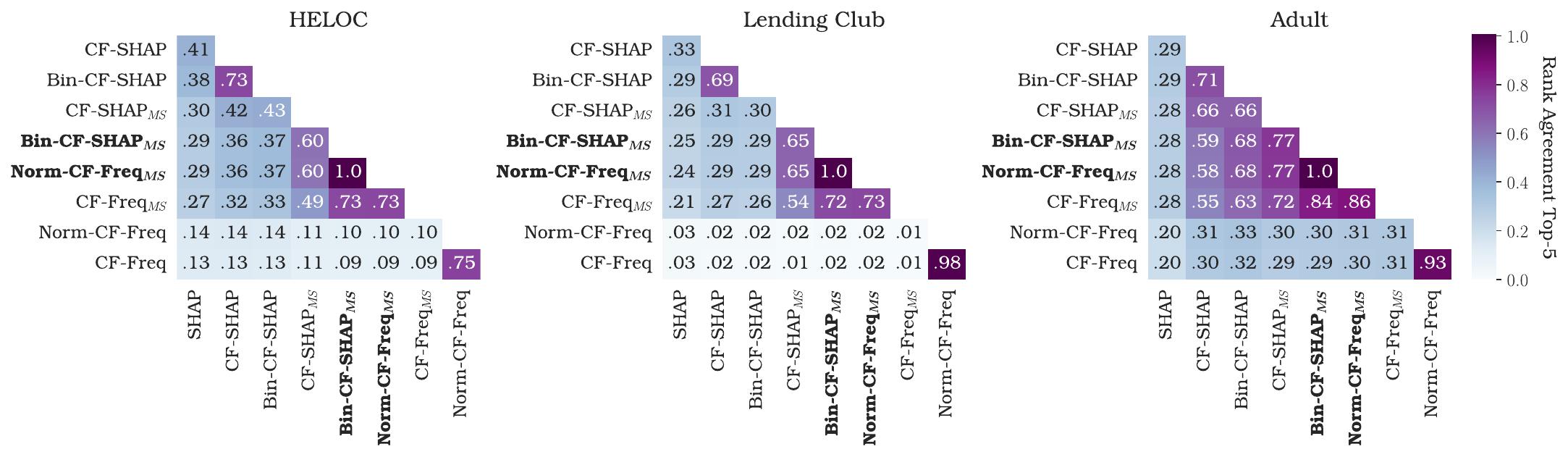}
         \caption{Top-5}
     \end{subfigure}
     \begin{subfigure}[b]{\textwidth}
         \centering
         \includegraphics[width=.99\textwidth]{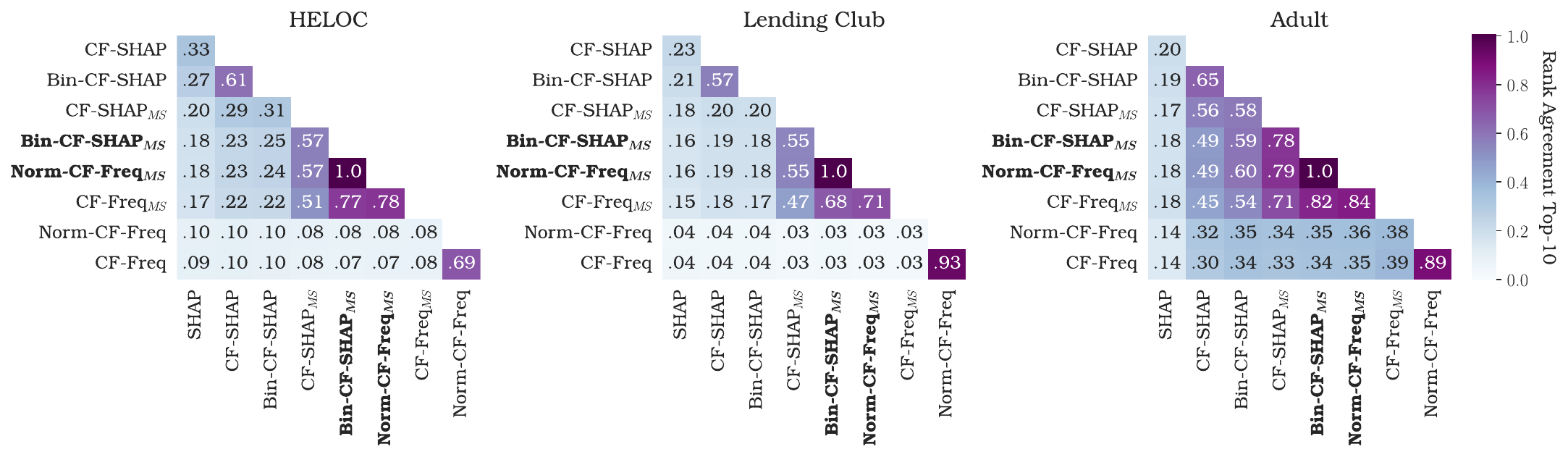}
         \caption{Top-10}
     \end{subfigure}
    \caption{Average pairwise Rank Agreement between explanations for different datasets. See \cref{appendix:additional_difference} for more details.}
    \label{fig:additional_difference_rank_agreement}
\end{figure*}

\begin{figure*}[!ht]
    \centering
     \includegraphics[width=.99\textwidth]{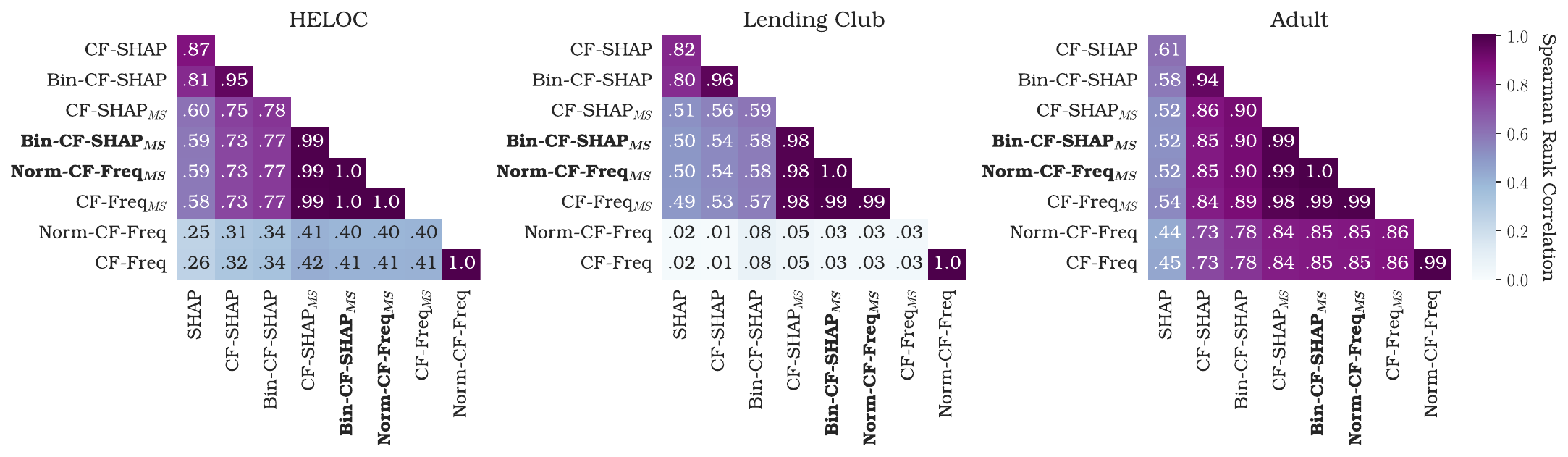}
    \caption{Average pairwise Spearman rank correlation between explanations for different datasets. See \cref{appendix:additional_difference} for more details.}
    \label{fig:additional_difference_spearman}
\end{figure*}

\begin{figure*}[!ht]
    \centering
     \includegraphics[width=.99\textwidth]{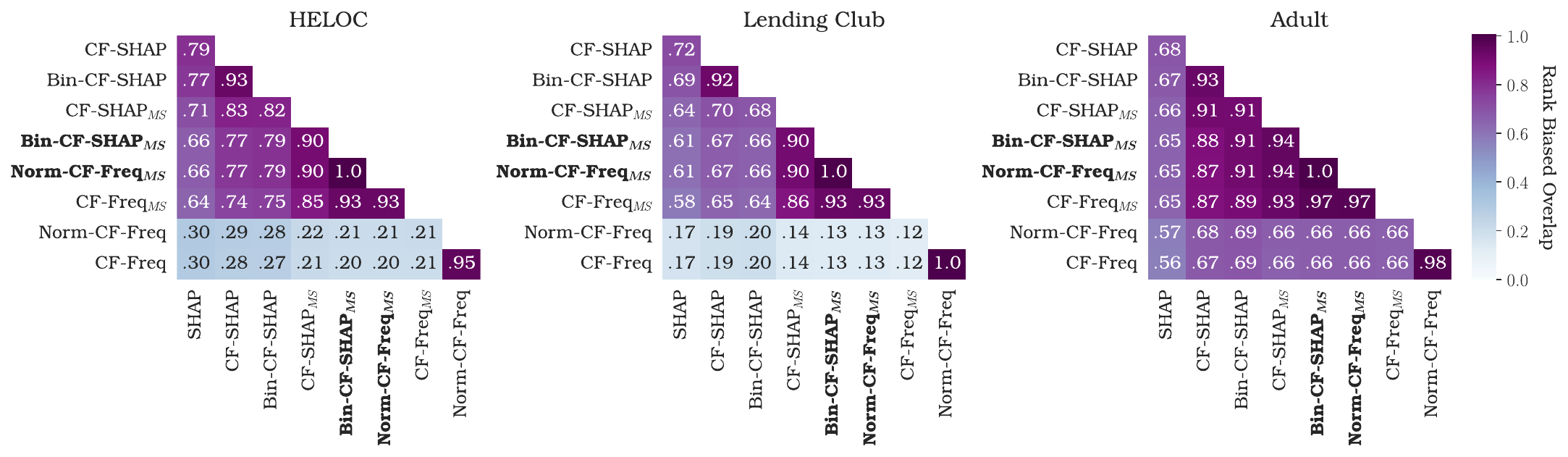}
    \caption{Average pairwise Rank Biased Overlap between explanations for different datasets. See \cref{appendix:additional_difference} for more details.}
    \label{fig:additional_difference_rbo}
\end{figure*}

\subsection{Counterfactual-ability and plausibility}

The counterfactual-ability and plausibility metrics proposed in \citet{Albini2022} have few hyper-parameters. In particular, they can be run using different strategies to induce a recourse from a feature importance explanations (called \emph{action functions}) and different ways to evaluate the cost of the recourse (called \emph{cost functions}). We refer the reader to \citet{Albini2022} for more details on the evaluation metrics and the hyper-parameters.

The results we reported in \cref{fig:experiment_performance} in the main text were obtained using \emph{random recourse} and \emph{total quantile shift cost}. To show the robustness of our evaluation under different action functions and cost functions we run the same experiments with the alternative definitions of cost and action functions that have been proposed in \cite{Albini2022}. 

In particular, in this appendix we report the results under the following alternative assumptions:
\begin{itemize}
\item {random recourse} and {quantile shift cost with L2 norm;}
\item {proportional recourse} and {total quantile shift cost};
\item {proportional recourse} and {total quantile shift cost under L2 norm}.
\end{itemize}

\textit{Results} - \cref{fig:additional_costs} shows the results which are indeed consistent with those presented in \cref{sec:experiments}.

\begin{figure*}
    \centering
     \begin{subfigure}[b]{\textwidth}
         \centering
         \includegraphics[width=.99\textwidth]{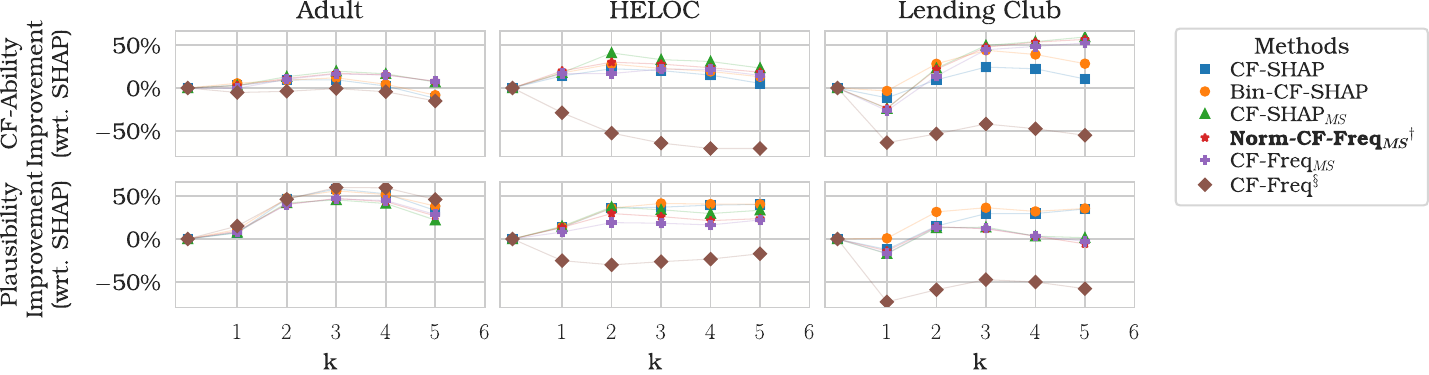}
         \caption{Random Recourse / L2 Cost Function}
     \end{subfigure}
     \hfill
     \begin{subfigure}[b]{\textwidth}
         \centering
         \includegraphics[width=.99\textwidth]{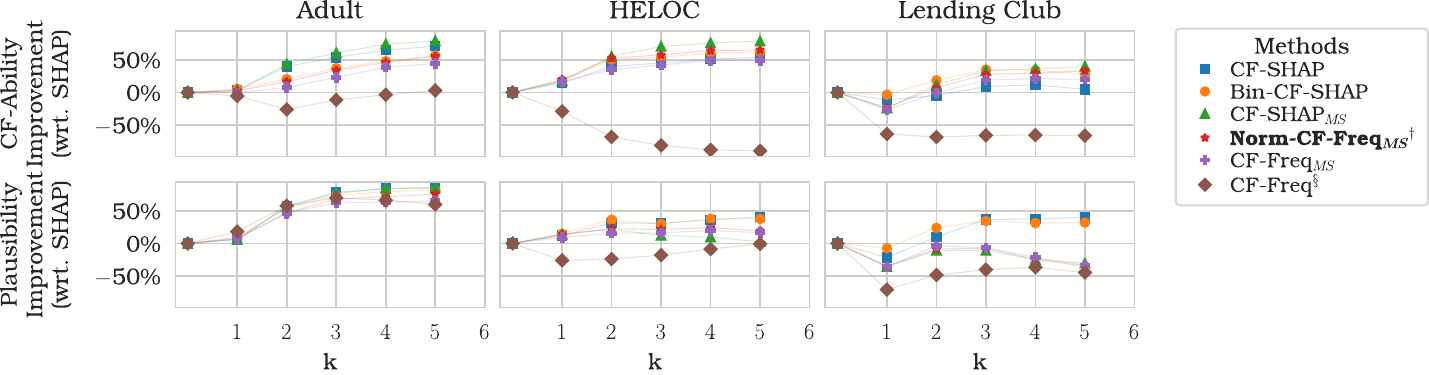}
         \caption{Proportional Recourse / L1 Cost Function}
     \end{subfigure}
     \hfill
     \begin{subfigure}[b]{\textwidth}
         \centering
         \includegraphics[width=.99\textwidth]{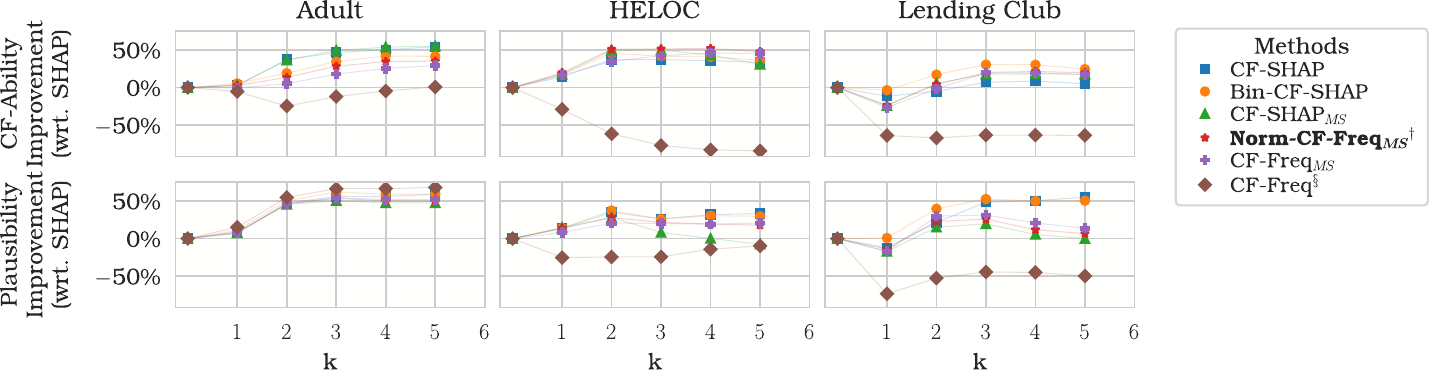}
         \caption{Proportional Recourse / L2 Cost Function}
     \end{subfigure}
    \caption{Counterfactual-ability and plausibility improvement (the higher the better) with respect to SHAP under different assumptions of recourse strategy and cost of counterfactuals. See \cref{appendix:additional_difference} for more details.}\label{fig:additional_costs}
\end{figure*}

\end{document}